\documentclass[11pt]{article}
\usepackage {jr,comment,enumerate, euscript,scalefnt,myfigures}
\usepackage[usenames]{color}
\usepackage[normalem]{ulem}

\definecolor{lightblue}{rgb}{0.8,0.9,1}

\title {Sparse Empirical Bayes Analysis (SEBA)}
\author{Natalia Bochkina  \& \jr}

\def\spect{\mathfrak{S}}
\def\be{\mathfrak{b}}
\def\t{^{\scriptstyle{\mathsf{T}}} \hspace{-0.00em}}

\def\O{\EuScript{O}}
\def\sfa{\scalefont{0.6}}
\def\sfb{\scalefont{1.6666666666}}
\def\o{\ensuremath{\text{\sfa$\O$\sfb}}}
\def\tihat{\dacc\ti\hat}

\def\calF{{\cal F}}
\def\rank{{\text{rank}}}

\begin{document}
\maketitle
\def\marg#1{}
\begin{abstract}
We consider a joint processing of  $n$ independent sparse
regression problems. Each is based on a sample
$(y_{i1},x_{i1})\dots,(y_{im},x_{im})$ of $m$ \iid observations
from $y_{i1}=x_{i1}\t\beta_i+\eps_{i1}$, $y_{i1}\in \R$, $x_{i
1}\in\R^p$, $i=1,\dots,n$, and $\eps_{i1}\dist N(0,\sig^2)$, say.
$p$ is large enough so that the empirical risk minimizer is not
consistent. We consider three possible extensions of the lasso
estimator to deal with this problem, the lassoes, the group lasso
and the RING lasso, each utilizing a different assumption how
these problems are related. For each estimator we give a Bayesian
interpretation, and we present both persistency analysis and
non-asymptotic error bounds based on restricted eigenvalue - type
assumptions.
\end{abstract}


``\dots and only a star or two set sparsedly in the vault of heaven; and you
will find a sight as stimulating as the hoariest summit of the Alps." R. L.
Stevenson  


\section{Introduction}

We consider the model
 \eqsplit[basicModel]{
    Y_i=X_i\t\beta_i+\eps_i, \quad i=1,\dots,n,
  }
or more explicitly
 \eqsplit{
    y_{ij} = x_{ij}\t \beta_i + \eps_{ij}, \quad i=1,\dots,n, \;j=1,\dots,m
  }
where \(\beta_i\in\R^p\), \(X_{i}\in\R^{m\times p}\) is either
 deterministic fixed design matrix, or a sample of $m$ independent \(\R^p\)
random vectors.
 Generally, we think of $j$ indexing replicates (of similar items within the group) and
$i$ indexing groups (of replicates).
 Finally,
$\eps_{ij}$, \(i=1,\dots,n, \;j=1,\dots,m\) are (at least
uncorrelated with the \(x\)s), but typically assumed to be \iid
sub-Gaussian random variables, independent of the regressors
$x_{ij}$. We can consider this as $n$ partially related regression models, with $m$ \iid observations on the each model. For simplicity, we assume that all variables have
expectation 0. The fact that the number of observations does not
dependent on $i$ is arbitrary and is assumed only for the sake
of notational simplicity.

The standard FDA (functional data analysis) is of this form, when
the functions are approximated by their projections on some basis. Here we have $n$ \iid random functions, and each group  can be considered as $m$ noisy observations, each one is  on the value of these functions at a given value of the argument.  Thus, \eqsplit[fda2]{
    y_{ij} = g_i(z_{ij}) + \eps_{ij},
    }
where $z_{ij}\in[0,1]$. The model fits the
regression setup of \eqref{basicModel}, if $g(z)=\summ \ell1p\beta_\ell h_\ell(p)$  where $h_1,\dots,h_p$ are in $L_2(0,1)$, and $x_{ij\ell}=h_\ell(z_{ij})$.

This approach is in the spirit of the empirical Bayes approach (or compound decision theory, note however that the term ``empirical Bayes'' has a few other meanings in the literature), cf,
\cite{Robbins1,Robbins2,Zhang-EB}. The empirical Bayes to sparsity was considered before, e.g.,
\cite{Zhang-W,BG-EB, GPR, GR-EB}. However, in these discussions
the compound decision problem was within a single vector, while we
consider  the compound decision to be between the vectors, where the vectors are
the basic units. The beauty of the concept of compound decision, is that we do not have to assume that in reality the units are related. They are considered as related only because our loss function is additive.

One  of the standard tools for finding sparse solutions in a large
$p$ small $m$ situation is the lasso
(Tibshirani~\cite{Tibsh-Lasso}), and the methods we consider are
its extensions.

We will make use of the following notation. Introduce $l_{p,q }$
norm  of a set of vectors $z_1,\dots,z_n$, not necessarily of the
same length, $z_{ij}$, $i=1,\dots,n$, $j=1,\dots,J_i$:
\begin{definition}
$||z||_{p,q} = \left[\summ i1n \left(\sum_{j\in  J_i} |z_{ij}|^p
\right)^{q/p} \right]^{1/q}.$
\end{definition}
These norms will serve as a penalty on the size of the  matrix $\scb=(\beta_1,\dots,\beta_n)$. Different norms imply different estimators, each appropriate under different assumptions.

Within the framework of the compound decision theory, we can have
different scenarios, and we consider three of them. In Section \ref{sec:lassoes} we investigate the situation when there is no  direct relationship between
the groups, and the only way the data are combined together is via
the selection of the common penalty. In this case the sparsity
pattern of the solution for each group are unrelated.
We argue that the alternative formulation of the lasso procedure in terms
of $\ell_{2,1}$ (or, more generally, $\ell_{\alpha, 1}$) norm which we refer to as ``lassoes''  can be more natural
than the simple lasso, and this is argued from different points of view.

The motivation is as follows. The lasso method can be described in
two related ways. Consider the one group version,
$y_j=x_j\t\beta+\eps_j$. The lasso estimator can be defined by
 \eqsplit{
    \text{Minimize } \summ j1m  (y_j-x_j\t\beta)^2 \quad\text{s.t.}\quad \|\beta\|_1< A.
  }
An equivalent definition, using Lagrange multiplier is given by
 \eqsplit{
    \text{Minimize} \summ j1m (y_j-x_j\t\beta)^2 +\lm \|\beta\|_1^\al,
  }
where  \al can be any arbitrarily chosen positive number.  In the
literature one can find almost only $\al=1$. One exception is
Greenshtein and Ritov \cite{p:GR-persist} where $\al=2$ was found
more natural, also it was just a matter of aesthetics.  We would argue that $\al>2$ may be more intuitive.
Our first algorithm generalizes this representation of the lasso directly to deal
with compound model \eqref{basicModel}.

In the framework of the compound decision problem it is possible to consider
the $n$ groups as repeated similar models for $p$ variables, and to choose the
variables that are useful for all models. We consider this in Section \ref{sec:grouped}. The relevant variation
of the lasso procedure in this case is group lasso introduced by
Yuan and Lin~\cite{p:GroupLasso-def}:
 \eqsplit[lassostar]{
    \text{Minimize} \summ i1n \summ j1m (y_{ij}-x_{ij}\t\beta_i)^2 +\lm
    \|\beta\|_{2,1}.
  }
The authors  also showed that in this case the sparsity pattern of
variables is the same (with probability 1). Non-asymptotic inequalities under restricted eigenvalue type condition for group lasso are given by Lounici et
al.~\cite{p:Lounici-GroupLasso}.

Now, the standard notion  of sparsity, as captured by the $L_0$
norm, or by the standard lasso and group lasso, is basis
dependent.  Consider the model of \eqref{fda2}. If, for example,
$g(z)=\ind(a<z\le b)$, then this example is sparse when
$h_\ell(z)=\ind(z>\ell/p)$. It is not sparse if
$h_\ell(z)=(z-\ell/p)^+$. On the other hand, a function $g$ which
has a piece-wise constant slope is sparse in the latter basis, but
not in the former, even though, each function can be represented
equally well in both bases.

 Suppose that there is a sparse representation in some
unknown basis, but assumed common to the $n$ groups. The
question arises, can we recover the basis corresponding to the sparsest representation? We will argue that
this penalty, also known as  trace norm or Schatten norm with
$p=1$, aims in finding the rotation that gives the best sparse
representation of all vectors instantaneously (Section
\ref{sec:spectral}). We refer to this method as the
rotation-invariant lasso, or shortly as the RING lasso. This is not surprising as under some
conditions, this penalty also solves the minimum rank problem (see
Candes and Recht~\cite{Candes-Recht} for the noiselss case, and
Bach~\cite{Bach-Trace} for some asymptotic results).  By analogy
with the lassoes argument, a higher power of the trace norm as a
penalty may be more intuitive to a Bayesian.

For both procedures considered here, the lassoes and the
RING lasso, we present the bounds on their
persistency as well as non-asymptotic inequalities under restricted eigenvalues type condition. All the proofs are
given in the Appendix.

\section{The lassoes procedure }
\label{sec:lassoes}

The  minimal structural relationship we may assume is that the
$\beta's$ are not related, except that we believe that there is a
bound on the average sparsity of the $\beta$'s. One possible
approach would be to consider the problem as a standard sparse
regression problem with $nm$ observations, a single vector of
coefficients $\beta=(\beta_1\t,\dots,\beta_n\t)\t$, and a block
diagonal design matrix $X$. This solution imposes very little on
the similarity among $\beta_1,\dots,\beta_n$. The lassoes
procedure discussed in this section assume that these vectors are
similar, at least in their level of sparsity.

\subsection{Prediction error minimization}

In this paper we adopt an oracle point of view. Our estimator is  the empirical minimizer of the risk penalized by the complexity of the solution (i.e., by its $\ell_1$ norm). We compare this estimator to the solution of an ``oracle'' who does the same, but optimizing over the true, unknown to simple human beings, population distribution.

We assume that each
vector of $\beta_i$, $i=1,\dots,n$, solves a different problem, and these problems
are related only through the joint loss function, which is the sum
of the individual losses. To be clearer, we assume that for each
$i=1,\dots,n$,   $z_{ij}=(y_{ij}, x_{ij}\t)\t$, $j=1,\dots,m$ are
\iid, sub-Gaussian random variables, drawn from a distribution
$Q_i$. Let $z_i=(y_i,x_i\t)\t$ be an independent sample from
$Q_i$. For any vector $a$, let  $\ti a=(-1,a\t)\t$, and let
$\tilde\Sig_i$ be the covariance matrix of $z_i$ and
$\mathfrak{S}=(\ti\Sig_1,\dots,\ti\Sig_n)$. The goal is to find
the matrix $\hat\scb=(\hat\beta_1,\dots,\hat\beta_n)$ that
minimizes the mean prediction error:
 \eqsplit[persis]{
    L(\scb,\mathfrak S) &= \summ i1n \E_{Q_i}(y_i-x_i\t\beta_i)^2=\summ i1n \ti\beta_i\t\ti\Sig_i\ti\beta_i.
  }
For $p$  small, the natural approach is empirical risk
minimization, that is replacing $\ti\Sig_i$ in \eqref{persis} by
$\ti S_i$, the empirical covariance matrix of $z_i$. However,
generally speaking, if $p$ is large, empirical risk minimization
results in  overfitting the data. Greenshtein and Ritov
\cite{p:GR-persist} suggested (for the standard $n=1$) minimization over a restricted set
of possible $\beta$'s, in particular, to either  $L_1$ or $L_0$
balls. In fact, their argument is based on the following
simple observations
 \eqsplit[roughapprox]{
    \bigl| \ti\beta\t(\ti\Sig_i - \ti S_i)\ti\beta\bigr|
    &\leq \| \ti\Sig_i - \ti S_i \|_\en \|\ti\beta\|_1^2
    \\
    &\hspace{-10em}\text{and}
    \\
    \| \ti\Sig_i - \ti S_i \|_\en &=\O_p(m^{-1/2}\log p)
  }
(see Lemma~\ref{lem:persist} in the Appendix for the formal argument.)

This leads to  the natural extension of the single vector lasso to
the compound decision problem set up, where we penalize by the sum
of the \emph{squared} $L_1$ norms of vectors
$\ti\beta_1,\dots,\ti\beta_n$, and obtain the estimator defined
by:
 \eqsplit[lassoes]{
    (\dacc\ti\hat\beta_i,\dots,\dacc\ti\hat\beta_n)
    &=\argmin_{\ti\beta_1,\dots,\ti\beta_n} \Bigl\{m\summ i1n \ti\beta_i\t\ti S_i\ti\beta_i
    + \lm_n \summ i1n \|\ti\beta_i\|_1^2\Bigr\}
    \\
    &=\argmin_{\ti\beta_1,\dots,\ti\beta_n} \summ i1n \Bigl\{ \summ j1m (y_{ij}-x_{ij}\t\beta_i)^2
    + \lm_n \|\ti\beta_i\|_1^2\Bigr\}.
  }

The prediction error of the lassoes estimator can be bounded in
the following way. In the statement of the theorem, $c_n$ is the minimal achievable risk, while $C_n$ is the risk achieved by a particular sparse solution.
\begin{theorem}
\label{th:lassoes1} Let $\beta_{i0}$, $i=1,\dots,n$ be $n$
arbitrary vectors and let $C_n=n^{-1}\summ i1n
\ti\beta_{i0}\t\ti\Sig_i\ti\beta_{i0}$. Let $c_n=n^{-1}\summ i1n
\min_\beta\ti\beta\t\ti\Sig_i\ti\beta$.  Then
 \eqsplit{
     \summ i1n \tihat\beta_{i}\t\ti \Sig_i \tihat\beta_{i}
    \leq
     \summ i1n \ti\beta_{i0}\t\ti \Sig_i \ti\beta_{i0}
    + (\frac{\lm_n} m+\del_n) \summ i1n \|\ti\beta_{i0}\|_1^2
    - (\frac {\lm_n} m-\del_n) \summ i1n \|\tihat\beta_{i}\|_1^2,
  }
where $\del_n=\max_i\|\ti S_i-\it\Sig_i\|_\en$.  If also $\lm_n/m\to0$ and $\lm_n/(m^{1/2}\log(np))\to\en$, then
 \eqsplit[bndthb]{
    \summ i1n\|\tihat\beta_i\|_1^2
    =\O_p\bigl(mn\frac{C_n-c_n}{\lm_n}\bigr) + \bigl(1+\O(\frac{m^{1/2}}{\lm_n}\log(np))\bigr)
    \summ i1n\|\ti\beta_{i0}\|_1^2
  }and
 \eqsplit{
     \summ i1n \tihat\beta_{i}\t\ti \Sig_i \tihat\beta_{i}
    \leq
     \summ i1n \ti\beta_{i0}\t\ti \Sig_i \ti\beta_{i0}
    + \bigl(1+\o_p(1)\bigr)\frac{\lm_n} m \summ i1n \|\ti\beta_{i0}\|_1^2.
  }
\end{theorem}

The result is meaningful, although not as strong as may be wished,
as long as $C_n-c_n\to 0$, while $n^{-1}\summ i1n
\|\ti\beta_{i0}\|_1^2=\o_p(m^{1/2})$.  That is, when there is a relatively sparse approximations to the best regression functions. Here sparse means only that the  $L_1$ norms of vectors is strictly smaller, on the average, than $\sqrt m$. Of course, if the minimizer
of $\ti\beta\t\ti\Sig_i\ti\beta$ itself is sparse, then by
\eqref{bndthb} $\tihat\beta_1,\dots,\tihat\beta_n$ are as sparse
as the true minimizers .

Also  note, that the prescription that the theorem gives for
selecting $\lm_n$, is sharp: choose $\lm_n$ as close as possible
to $m\del_n$, or slightly larger than $\sqrt m$.

\subsection{A Bayesian perspective}
The estimators $\tihat\beta_1,\dots,\tihat\beta_m$ look as if they
are the mode of the a-posteriori distribution of the $\beta_i$'s
when $y_{ij}|\beta_i\dist N(x_{ij}\t\beta_i,\sig^2)$, the
$\beta_1,\dots,\beta_n$ are a priori independent, and $\beta_i$
has a prior density proportional to
$\exp(-\lm_n\|\ti\beta_i\|_1^2/\sig^2)$. This distribution can be
constructed  as follows. Suppose $T_i\dist N(0,\lm_n^{-1}\sig^2)$.
Given $T_i$, let $u_{i1},\dots,u_{ip}$ be distributed uniformly on
the simplex $\{u_{i\ell}\geq 0, \summ \ell1n u_{i\ell}=|T_i|\}$.
Let $s_{i1},\dots,s_{ip}$ be \iid Rademacher random variables
(taking values $\pm 1$ with probabilities $0.5$), independent of
$T_i,u_{i1},\dots,u_{ip}$. Finally let $\beta_{i\ell}=u_{i\ell}
 s_{i\ell} $, $\ell=1,\dots,p$.

However, this Bayesian  point of view is not consistent with the conditions
of Theorem \ref{th:lassoes1}. An appropriate prior should express the beliefs on the unknown parameter which are by definition conceptually independent of the amount data to be collected. However, the permitted range of $\lm_n$ does
not depend on the assumed range of $\|\ti\beta_i\|$, but quite
artificially should be in order between $m^{1/2} $ and $m$.
That is, the penalty should be increased with the number of
observations on $\beta_i$, although in a slower rate than $m$. In
fact, even if we relax what we mean by ``prior'', the value of $\lm_n$ goes in the `wrong' direction.  As $m\to\en$, one may
wish to use weaker a-priori  assumptions, and permits  $T$ to have
a-priori second moment going to infinity, not to 0, as entailed by
$\lm_n\to 0$.

 We would like to consider a more general penalty of the form $\summ i1n \|\beta_i\|_1^\al$. A power $\al\neq 1$ of $\ell_1$ norm of $\beta$ as a
penalty introduces a priori dependence between the variables which
is not the case for the regular lasso penalty with $\al=1$, where
all $\beta_{ij}$ are a priori independent. As $\al$ increases, the
sparsity of the different vectors tends to be the same. Note that
given the value of $\lm_n$, the $n$ problems are treated
independently.  The compound decision problem is reduced to
picking a common level of penalty. When this choice is data based,
the different vectors become dependent. This is the main benefit
of this approach---the selection of the regularization is based on
all the $mn$ observations.

For a proper Bayesian perspective, we need to consider a prior with much smaller tails than the
normal. Suppose for simplicity that $c_n=C_n$ (that is, the ``true'' regressors are sparse), and
$\max_i\|\beta_{i0}\|_1<\en$.

\begin{theorem}
\label{th:lassoes2}
Let $\beta_{i0}$ be the minimizer of $\ti\beta\t\Sig_i\ti\beta$. Suppose $\max_i\|\beta_{i0}\|_1<\en$. Consider the estimators:
 \eqsplit{
    (\dacc\ti\hat\beta_i,\dots,\dacc\ti\hat\beta_n)
    &=\argmin_{\ti\beta_1,\dots,\ti\beta_n} \Bigl\{m\summ i1n \ti\beta_i\t\ti S_i\ti\beta_i
    + \lm_n \summ i1n \|\ti\beta_i\|_1^\al\Bigr\}
  }
for some $\al>2$. Assume that $\lambda_n =\O(m \delta_m) =
\O(m^{1/2} \log p)$. Then
 \eqsplit{
    n^{-1}\summ i1n \|\tihat\beta_i\|_1^2 &= \O((m\del_n/\lm_n)^{2/(\al-2)} ),
  }
and
 \eqsplit{
    \summ i1n \tihat\beta_{i}\t\ti \Sig_i \tihat\beta_{i}
    &\leq   \summ i1n \ti\beta_{i0}\t\ti \Sig_i \ti\beta_{i0}
    + \O_p(n (m/\lm_n)^{2/(\al-2)} \del_n^{\al/(\al-2)}).
  }

\end{theorem}

\begin{remark}\label{rem:lassoes2} If the assumption $\lambda_n =\O(m \delta_m)$ does not hold, i.e. if $m \delta_m/\lambda_n =\o(1)$, then the
error term dominates the penalty and we get similar rates as in
Theorem~\ref{th:lassoes1}, i.e.
 \eqsplit{
    n^{-1}\summ i1n \|\tihat\beta_i\|_1^2 &= \O(1),
  }
and
 \eqsplit{
    \summ i1n \tihat\beta_{i}\t\ti \Sig_i \tihat\beta_{i}
    &\leq  \summ i1n \ti\beta_{i0}\t\ti \Sig_i \ti\beta_{i0}
    + \O_p\left( n\lm_n/m\right).
  }

\end{remark}

 Note that we can take in fact $\lm_n\to 0$, to accommodate an increasing value of the $\tihat\beta_i$'s.

The  theorem suggests a simple way to select $\lm_n$ based on the
data. Note that $n^{-1}\summ i1n \|\tihat\beta_i\|_1^2$ is a
decreasing function of $\lm$. Hence, we can start with a very
large value of $\lm$ and decrease it until $n^{-1}\summ i1n
\|\tihat\beta_i\|_1^2\approx \lm^{-2/\al}$.



\subsection{Restricted eigenvalues conditions and non-asymptotic inequalities}

Before stating the conditions and the inequalities for the lassoes procedure, we
introduce some notation and definitions.

For a vector $\beta$, let $\scm(\beta)$ be the cardinality of its
support: $\scm(\beta)=\sum_i\ind(\beta_i\ne0)$. Given a matrix
$\Delta\in \mathbb{R}^{n \times p}$ and  given a set $J=\{J_i\}$,
$J_i \subset \{1,\dots,  p\}$, we denote $\Delta_J =
\{\Delta_{i,j}, \, i=1,\dots,n, \, j\in J_i\}$. By the complement
$J^c$ of $J$ we denote the set $\{J_1^c, \dots, J_n^c\}$, i.e. the
set of complements of $J_i$'s. Below, $X$ is $np \times m$ block
diagonal design matrix, $X = \diag (X_1, X_2, \dots, X_n)$, and
with some abuse of notation, a matrix $\Delta=(\Delta_1,\dots,\Delta_n)$ may be considered as the vector $(\Delta_1\t,
\dots,\Delta_n\t)\t$. Finally, recall the notation $\scb=(\beta_1,\dots,\beta_n)$


The restricted eigenvalue assumption of Bickel et al.~\cite{p:BRT}
(and Lounici et al.~\cite{p:Lounici-GroupLasso}) can be
generalized to incorporate unequal subsets $J_i$s. In the
assumption below, the restriction is given in terms of $\ell_{q,
1}$ norm, $q\geqslant 1$. 

\vspace{\bigskipamount}\par \noindent{\bf Assumption}
RE$_q(s,c_0,\kappa)$.
$$
\kappa = \min\left\{ \frac{|| X\t \Delta||_{2
}}{\sqrt{m}||\Delta_J||_{ 2}}: \, \max_i |J_i| \leqslant s, \,
\Delta\in \mathbb{R}^{n \times p}\setminus \{0\}, \,
||\Delta_{J^c}||_{q,1} \leqslant c_0 ||\Delta_{J}||_{q,1} \right\}
> 0.
$$
We apply it with $q=1$, and in  Lounici et
al.~\cite{p:Lounici-GroupLasso} it was used for $q=2$.
 We call it a {\it restricted eigenvalue assumption} to be consistent
with the literature. In fact, as stated it is a definition of
$\kappa$ as the maximal value that satisfies the condition, and
the only real assumption is that $\kappa$ is positive. However,
the larger $\kappa$ is, the more useful the ``assumption'' is.
Discussion of the normalisation by $\sqrt{m}$ can be found in
Lounici et al.~\cite{p:Lounici-GroupLasso}.


For penalty $\lambda \sum_i ||\beta_i||_1^{\alpha}$, we have the
following  inequalities.

\begin{theorem}
\label{th:LassoL1p} Assume $y_{ij} \sim \scn(x_{ij}\t\beta_i,
\sigma^2)$,  and let \(\hat\beta\) be a minimizer of
\eqref{lassoes}, with
$$  \lambda \geqslant \frac{4A\sig
\sqrt{m\log(np)}}{  \alpha \max(B^{\alpha-1},
\hat{B}^{\alpha-1})},$$ where $\alpha\geqslant 1$ and
$A>\sqrt{2}$, $B \geqslant \max_i ||\beta_i||_1$ and
 $\hat{B}\geqslant \max_i ||\hat\beta_i||_1$, $\max(B,\hat{B})>0$ ($B$ may depend
on $n,m,p$, and so can $\hat{B}$). Suppose that generalized
assumption RE$_1(s, 3, \kappa)$ defined above holds, $\sum_{j=1}^m
x_{ij\ell}^2 = m$ for all $i,\ell$, and $\scm(\beta_i) \leqslant
s$ for all $i$.

Then, with probability at least $1 -(np)^{1-A^2/2}$,
\begin{enumerate}

\item[(a)] The root means squared prediction error is bounded by:$$
  \frac 1 {\sqrt{nm}}||X\t(\hat\scb - \scb)||_2
 \leqslant  \frac{ \sqrt{s }}{\kappa \sqrt{m} }
\left[ \frac {3 \alpha \lambda}{2\sqrt{ m}}   \max(B^{\alpha-1},
  \hat{B}^{\alpha-1}) + 2A \sig \sqrt{ \log(np) }\right], $$

\item[(b)] The mean estimation absolute error is bounded by:
$$ \frac 1 n ||\scb - \hat\scb||_1  \leqslant
 \frac{4s}{m \kappa^2} \left[ \frac {3\alpha \lambda}{2}
 \max(B^{\alpha-1},
\hat{B}^{\alpha-1}) + 2A \sig \sqrt{m\log(np)}\right],$$

\item[(c)] If \, $| ||\hat\beta_i||_1^{\al - 1}   -
    b^{\alpha-1}/2)| \geqslant  4\delta/b^{\alpha-1}$ for some
    $\delta>0$,

$$
  \scm(\hat\beta_i) \leq \|X_i(\beta_i-\hat\beta_i)\|_2^2
    \frac{m  \phi_{i,\, \max}} {\left(\lm  \alpha  ||\hat\beta_i||_1^{\alpha-1}/2- A  \sig \sqrt{m\log(np)}\right)^2
    },
$$ where $\phi_{i, \max}$ is the maximal eigenvalue of $X_{i}\t
X_i/m$.


\end{enumerate}
\end{theorem}

Note that for $\alpha = 1$, if we take $\lm = 2A\sig
\sqrt{m\log(np)}$, the bounds are of the same order as for the
lasso with $np$-dimensional $\beta$ ( up to a constant of 2, cf.
Theorem 7.2 in Bickel et al.~\cite{p:BRT}). For $\alpha>1$, we
have dependence of the bounds on the $\ell_1$ norm of $\beta$ and
$\hat\beta$.

We can use bounds on the norm of $\hat\beta$ given in
Theorem~\ref{th:lassoes2} to obtain the following results.

\begin{theorem}\label{th:L12merge2}
Assume $y_{ij} \sim \scn(x_{ij}\t\beta_i, \sigma^2)$, with
$\max_{i} \|\beta_i\|_1  \leqslant b $ where $b>0$ can depend on
$n,m,p$. Take some $\eta \in (0,1)$. Let \(\hat\beta\) be a
minimizer of \eqref{lassoes}, with
$$  \lambda = \frac{4A\sig}{\al\, b^{\al-1}}\sqrt{m\log(np)},$$ $A>\sqrt{2}$,
such that $b > c \eta^{1/(2(\al-1))}$  for some constant $c>0$.
Also, assume that  $C_n - c_n = \O(m\delta_n)$, as defined in
Theorem~\ref{th:lassoes1}.

Suppose that generalized assumption RE$_1(s, 3, \kappa)$ defined
above holds, $\sum_{j=1}^m x_{ij\ell}^2 = m$ for all $i, \, \ell$,
and $\scm(\beta_i) \leqslant s$ for all $i$.

Then, for some constant $C>0$, with probability at least $1 -
\left(\eta+ (np)^{1-A^2/2}\right)$,
\begin{enumerate}

\item[(a)] The prediction error can be bounded by:
$$ ||X\t(\hat\scb - \scb)||_2^2
 \leqslant  \frac{4A^2 \sig^2  s n \log(np)} {\kappa^2   }
  \left[1+3 C \left(\frac{ b }{\sqrt{\eta}
}\right)^{(\al-1)/(\al -2)}
  \right]^2,$$

\item[(b)] The estimation absolute error is bounded by:
$$ ||\scb - \hat\scb||_1  \leqslant \frac{2A \sig s n
\sqrt{\log(np)}}{\kappa^2  \sqrt{m} } \left[1+3 C \left(\frac{ b
}{\sqrt{\eta} }\right)^{(\al-1)/(\al -2)}
  \right].$$

\item[(c)] Average sparsity of $\hat\beta_i$:
$$\frac 1 n  \summ i1n \scm(\hat\beta_i) \leqslant
\,s  \,
    \frac{4 \phi_{ \max}}{\kappa^2 \delta^2} \left[1+3 C \left(\frac{ b
}{\sqrt{\eta} }\right)^{1+1/(\al -2)}
  \right]^2,
$$
where $\phi_{\, \max}$ is the largest eigenvalue of $X\t
X/m$.

\end{enumerate}

\end{theorem}

This theorem also tells us how large $\ell_1$ norm of $\beta$ can
be to ensure good bounds on the prediction and estimation errors.

Note that under the Gaussian model and fixed design matrix,
assumption $C_n - c_n = \O(m\delta_n)$ is equivalent to
$||\scb||_2^2\leqslant C m\delta_n$.

\section{Group LASSO: Bayesian perspective}
\label{sec:grouped}

Group LASSO is defined (see Yuan and Lin~\cite{p:GroupLasso-def})
by
 \eqsplit[lasso]{
    (\hat\beta_1,\dots,\hat\beta_n) &=
    \argmin \Biggl[ \summ i1n \summ j1m (y_{ij} - x_{ij}\t\,\beta_i)^2 + \lambda
    \summ \ell1p \Bigl\{\summ i1n \beta_{i\ell}^2\Bigr\}^{1/2} \Biggr]
  }
Note that \((\hat\beta_1,\dots,\hat\beta_n)\) are defined as the minimum
point of a strictly convex function, and hence they can be found by equating
the gradient of this function to 0.

Recall the notation $\scb=(\beta_1,\dots,\beta_n)=(\be_1\t,\dots,\be_p\t)\t$. Note that \eqref{lasso} is equivalent to the mode of the
a-posteriori distribution when given $\scb$, $Y_{ij}$, $i=1,\dots,n$, $j=1,\dots,m$, are all independent, $y_{ij}\given\scb\dist
\normal(x_{ij}\t\,\beta_i,\sig^2)$, and a-priori, $\be_1,\dots,\be_p$, are \iid,  \eqsplit{
   f_\be(\be_\ell) \propto \exp\bigl\{-\ti\lambda \|\be_\ell\|_2\bigr\},\quad\ell=1,\dots,p,
 }
where $\ti\lm ={\lm}/(2\sigma^2)$. We consider now some property
of this prior. For each $\ell$, $b_\ell$ have a spherically
symmetric distribution. In particular they are uncorrelated and
have mean 0. However, they are not independent. Change of
variables to a polar system where \eqsplit{
    R_\ell &=\|\be_{\ell}\|_2
    \\
    \beta_{\ell i}  &= R w_{\ell i},\qquad w_{\ell}\in\bbs^{n-1},
    }
where $\bbs^{n-1}$ is the sphere in $\R^n$. Then, clearly,
\eqsplit[fgamma]{
    f(R_\ell,w_{\ell}) &= C_{n,\lm} R_\ell^{n-1} e^{-\ti\lm R_\ell}, \qquad R_\ell>0,
 }
where $C_{n,\,\lm} = {\ti\lm^n \Gamma(n/2)}/{2\Gamma(n)
\pi^{n/2}}$. Thus, $R_\ell,w_{\ell}$ are independent $R_\ell\dist
\Gamma (n,\ti\lambda)$, and $w_{\ell}$ is uniform over the unit
sphere.

The conditional distribution of one of the coordinates of
$\be_\ell$, say the  first, given the rest has the form
 \eqsplit{
    f(\be_{\ell 1} | \be_{\ell2},\dots,\be_{\ell n}, \summ i2n\be_{\ell i}^2 =\rho^2)
    &\propto e^{-\ti\lm\rho \sqrt{1+\be_{\ell 1}^2/\rho^2}}
    }
which for small $\be_{\ell 1}/\rho$ looks like the normal density
with mean 0 and variance $\rho/\ti\lm$, while for large $\be_{\ell
1}/\rho$ behaves like the exponential distribution with mean
$\ti\lm^{-1}$.

The sparsity property of the prior comes from the linear component
of log-density of $R$. If $\ti\lm$ is large and the $Y$s are
small, this component dominates the log-a-posteriori distribution
and hence the maximum will be at 0.

 Fix now \(\ell\in\{1,\dots,p\}\), and consider the estimating equation
for $\be_\ell$ --- the \(\ell\) components of the $\beta$'s. Fix the rest of the parameters and let \(\ti
Y_{ij\ell}^\scb = y_{ij}-\sum_{k\ne\ell}\beta_{ik}x_{ijk} \). Then
\(\hat\be_{\ell i}\), \(i=1,\dots,n\), satisfy
 \eqsplit{
    0&= -\summ j1m x_{ij\ell}(\ti Y^\scb_{ij\ell}-\hat\be_{\ell i}x_{ij\ell}) +
    \frac{\lambda \hat\be_{\ell i}}{\sqrt{\sum_k \hat\be_{\ell k}^2}},\qquad
    i=1,\dots,n   \\
    &= -\summ j1m x_{ij\ell}(\ti Y^\scb_{ij\ell}-\hat\be_{\ell i}x_{ij\ell})
    +\lambda^*_{\ell} \hat\be_{\ell i},\qquad\text{say} .
  }
Hence
 \eqsplit[betaell]{
    \hat\be_{\ell i} &= \frac{\summ j1m x_{ij\ell}\ti Y^{\scb}_{ij\ell}} {\lambda^*_{\ell}+\summ j1m
    x_{ij\ell}^2}.
  }
The estimator has an intuitive appeal. It is the least square estimator of
\(\be_{\ell i}\), \(\summ j1m x_{ij\ell}\ti Y^{\scb}_{ij\ell}/\summ j1m
x_{ij\ell}^2 \), pulled to 0. It is pulled less to zero as the variance of
\(\be_{\ell1},\dots,\be_{\ell n}\) increases (and \(\lambda^*_\ell\) is
getting smaller), and as the variance of the LS estimator is lower (i.e.,
when \(\summ j1m x_{ij\ell}^2\) is larger).

If the design is well balanced, \(\summ j1m x_{ij\ell}^2\equiv
m\), then we can characterize  the solution as follows. For a
fixed $\ell$, \(\hat\be_{\ell 1},\cdot,\hat\be_{\ell n}\) are the
least square solution shrunk toward 0 by the same amount, which
depends only on the estimated variance of \(\hat\be_{\ell
1},\dots,\hat\be_{\ell n}\). In the extreme case, \(\hat\be_{\ell
1}=\dots =\hat\be_{\ell n}=0\), otherwise (assuming the error
distribution is continuous) they are shrunken toward 0, but are
different from 0.

We can use \eqref{betaell} to solve for \(\lambda^*_{\ell}\)
 \eqsplit{
    \Bigl(\frac{\lambda}{\lambda^*_{\ell}}\Bigr)^2 &= \| \hat\be_{\ell}\|_2^2
    = \summ i1n \Biggl( \frac{\summ j1m x_{ij\ell}\ti Y^{\scb}_{ij\ell}} {\lambda^*_{\ell}+\summ j1m
    x_{ij\ell}^2}\Biggr)^2.
    }
Hence \(\lambda^*_{\ell}\) is the solution of
 \eqsplit[lmds]{
    \lambda^2 &= \summ i1n \Biggl( \frac{\lambda^*_{\ell}\summ j1m x_{ij\ell}\ti Y^{\scb}_{ij\ell}} {\lambda^*_{\ell}+\summ j1m
    x_{ij\ell}^2}\Biggr)^2.
  }
Note that the RHS is monotone increasing, so \eqref{lmds} has at most a
unique solution. It has no solution if at the limit
\(\lambda^*_{\ell}\to\en\), the RHS is still less than \(\lambda^2\). That is
if
 \eqsplit{
    \lambda^2 &> \summ i1n \Bigl( \summ j1m x_{ij\ell}\ti Y^{\scb}_{ij\ell}
    \Bigr)^2
  }
then \(\hat\be_{\ell}=0\). In particular if
 \eqsplit{
    \lambda^2 &> \summ i1n \Bigl( \summ j1m x_{ij\ell} Y_{ij\ell}
    \Bigr)^2,\qquad\ell=1,\dots,p
  }
Then all the random effect vectors are 0. In the balanced case the
RHS is $\O_p(mn\log(p))$. By  \eqref{fgamma},  this means that if
we want that the estimator will be 0 if the underlined true
parameters are 0, then the prior should prescribe that $\be_\ell$
has norm which is $\o(m^{-1})$. This conclusion is supported by
the recommended value of $\lm$ given, e.g. in
\cite{p:Lounici-GroupLasso}. 

Non-asymptotic inequalities and prediction properties of the group lasso
estimators  under restricted eigenvalues conditions are given in \cite{p:Lounici-GroupLasso}.

\section{The RING lasso}
\label{sec:spectral}

The rotation invariant group (RING) lasso is suggested as a natural extension of the group lasso to the situation where the proper sparse description of the regression function within a given basis is not known in advance. For example, when we prefer to leave it a-priori open whether the function should be described in terms of the standard Haar wavelet basis, a collection of interval  indicators, or a collection of step functions. All these three span the same linear space, but the true functions may be sparse in only one of them.

\subsection{Definition}

Let \(A=\sum c_i x_ix_i\t\), be a positive semi-definite matrix,
where \(x_1,x_2,\dots\) is an orthonormal basis of eigenvectors.
Then, we define \(A^{\gamma}=\sum c_i^{\gamma} x_ix_i\t\). We
consider now as penalty the function
 \eqsplit{
    |||\scb|||_1 = \trace\Bigl\{\bigl(\summ i1n
    \beta_i\beta_i\t\bigr)^{1/2}\Bigr\},
  } where
  $\scb=(\beta_1,\dots,\beta_n)=(\be_1\t,\dots,\be_p\t)\t$. This
  is also known as trace norm or Schatten norm with $p=1$.
  Note that \(|||\scb|||_1=\sum c_i^{1/2}\) where \(c_1,\dots,c_p\)
  are the eigenvalues of \(\scb\scb\t=\summ i1n \beta_i\beta_i\t\) (including
  multiplicities), i.e. this is the $\ell_1$ norm on the singular
  values of \scb. $|||\scb|||_1$ is a convex function of \scb.

In this section we study the estimator defined by
 \eqsplit[specPen]{
    \hat\scb&= \argmin_{\scb\in\R^{p\times n}}\{\summ i1n (y_{ij}-x_{ij}\t\beta_i)^2+
    \lm|||\scb|||_1.\}
  }
We refer to this problem as RING (Rotation INvariant Group) lasso.

The lassoes penalty considered primary the columns of \scb.  The
main focus of the group lasso was the rows. Penalty $|||\scb|||_1$
is symmetric in its treatment of the rows and columns since
$\spect\scb =\spect\scb\t$, where $\spect A$ denotes the spectrum
of $A$. Moreover, the penalty is invariant to the rotation of the
matrix \scb. In fact, \(|||\scb|||_1 = |||T\scb U |||_1\), where
\(T\) and \(U\) are \(n\times n\) and \(p\times p\) rotation
matrices:
 \eqsplit{
    (T\scb U)\t (T\scb U) &= U\t\scb\t\scb U
  }
and the RHS have the same eigenvalues as \(\scb\t\scb=\sum \beta_i\beta_i\t\).

The rotation-invariant penalty aims at finding a basis in which
$\beta_1,\dots,\beta_n$ have the same pattern of sparsity. This is
meaningless if $n$ is small --- any function is well approximated by the span of the basis is sparse in under the right rotation. However, we will argue that this can be done when $n$ is
large.

The following lemma describes a relationship between group lasso
and RING lasso.

\begin{lemma}
\label{lem:spectVsGroup}
\mbox{}\par

\begin{enumerate}[(i)]
\item $\|\scb\|_{2,1} \ge \inf_{U\in\scu} \|U \scb\|_{2,1}
=|||\scb|||_1$, where $\scu$ is the set of all unitary matrices.
\item There is a unitary matrix $U$, which may depend on the data,
such that if $X_1,\dots,X_n$ are rotated by $U\t$, then the
solution of the RING lasso \eqref{specPen} is the solution of the
group lasso in this basis.
\end{enumerate}

\end{lemma}

\subsection{The estimator  }

Let $\scb=\summ \xi1{p\wedge n}\al_\xi \beta_\xi^*{\be_\xi^*}\t$
be the singular value decomposition,  or the PCA, of \scb:
$\beta_1^*,\dots,\beta_p^*$ and $\be_1^*,\dots,\be_n^*$ are
orthonormal sub-bases  of $\R^p$ and $\R^n$ respectively,
$\al_1\geq\al_2\geq\dots$, and
$\scb\scb\t\beta_\xi^*=\al_\xi^2\beta_\xi^*$,
$\scb\t\scb\be_\xi^*=\al_\xi^2\be_\xi^*$, $\xi=1,\dots,p\wedge n$.
Let $T=\summ \xi1{p\wedge n} e_\xi{\beta_\xi^*}\t$ (clearly,
$TT\t=I$).
Consider the parametrization of the problem in the  rotated
coordinates, $\ti x_{ij}=Tx_{ij}$ and  $\ti\beta_i=T\beta_i$. Then
geometrically the regression problem is invariant:
$x_{ij}\t\beta_i=\ti x_{ik}\t\ti\beta_i$, and
$|||\scb|||_1=\|\ti\scb\|_{2,1} $,
up to a modified regression matrix.

The  representation $\hat\scb = \summ \xi1s \al_\xi \beta_\xi^*
{\be_\xi^*}\t$ shows that the difficulty of the problem is the
difficulty of estimating $s(n+p)$ parameters with $nm$
observations. Thus it is feasible as long as $s/m\to0$ and
$sp/nm\to0$.

 We have

\begin{theorem}
\label{th:sparseSpect} Suppose $p<n$. Then the solution of the
RING lasso is given by $\summ \xi1s \beta^*_\xi{\be_\xi^*}\t$,
$s=s_\lm \leq p$, and $s_\lm\dec 0$ as $\lm\to\en$. If $s=p$ then
the gradient of the target function is given in a matrix form by
 \eqsplit{
    -2R+\lm (\hat\scb\hat\scb\t)^{-1/2}\hat\scb
  }
where
 \eqsplit{
    R=\Bigl(X_1\t(Y_1-X_1\hat\beta_1),\dots,X_n\t(Y_n-X_n\hat\beta_n)\Bigr).
  }
And hence
 \eqsplit{
    \hat\beta_i=\bigl(  X_i\t X_i + \frac\lm2 (\hat\scb\hat\scb\t)^{-1/2}\bigr)^{-1}X_i\t Y_i.
  }
That is, the solution of a ridge regression with adaptive weight.

More  generally, let $\hat\scb=\summ \xi1s \alpha_\xi
\beta_\xi^*{\be_\xi}\t$, $s<p$, where $\beta_1^*,\dots,\beta_p^*$
is an orthonormal base of $\R^p$. Then the solution satisfies
 \eqsplit{
    &{\beta_\xi^*}\t R = \frac\lm2 {\beta_\xi^*}\t   (\hat\scb\hat\scb\t)^{+1/2}\hat\scb, \quad \xi\leq s
    \\
    &|{\beta_\xi^*}\t R \be_\xi^*| \leq \frac\lm 2,  \qquad \qquad \qquad s<\xi\leq p.
  }
where for any positive semi-definite matrix $A$, $A^{+1/2}$ is the Moore-Penrose  generalized inverse of $A^{1/2}$.
\end{theorem}
Roughly speaking the following can be concluded from the  theorem.
Suppose the data were generated by a sparse model (in \emph{some}
basis). Consider the problem in the transformed basis, and let $S$
be the set of non-zero coefficients of the true model. Suppose
that the design matrix is of full rank within the sparse model:
$X_i\t X_i=\O(m)$, and  that \lm is chosen such that $\lm\gg
\sqrt{nm\log(np)}$.  Then the coefficients corresponding to $S$
satisfy
 \eqsplit{
   \hat\beta_{Si} &= \bigl(X_i\t X_i + \frac\lm2 (\hat\scb_S\hat\scb_S\t)^{1/2}\bigr)^{-1}X_i\t Y_i.
  }
 Since it is expected that  $\lm(\scb_S\scb_S\t)^{1/2}$ is only slightly larger than $\O(m\log(np))$, it is completely dominated
  by $X_i\t X_i$, and the estimator of this part of the model is consistent. On the other hand, the rows of $R$ corresponding
  to coefficient not in the true model are only due to noise and hence each of them is  $\O(\sqrt {nm})$.
  The factor of $\log (np)$ ensures that their maximal norm will be below $\lm/2$, and the estimator is consistent.

\subsection{Bayesian perspectives}

We consider now the penalty for $\beta_k$ for a fixed $k$. Let $A=n^{-1}\sum_{k\ne i} \beta_k\beta_k\t$, and write the spectral value decomposition $n^{-1}\summ k1n \beta_k\beta_k\t=\sum c_jx_jx_j\t$ where  $\{x_j\}$ is an orthonormal basis of eigenvectors.  Using Taylor expansion for not too big $\beta_i$,  we get
\eqsplit{
    \trace\bigl( (nA+\beta_i\beta_i\t)^{1/2}\bigr)
     &\approx \sqrt{n}\trace(A^{1/2}) + \summ j1p \frac{x_j\t \beta_i\beta_i\t x_j}{2 c_j^{1/2}}
     \\
    &= \sqrt{n}\trace(A^{1/2}) + \frac 12 \beta_i\t \bigl(\sum c_j^{-1/2} x_jx_j\t\bigr)\beta_i
    \\
    &= \sqrt{n}\trace(A^{1/2}) + \frac12 \beta_i\t A^{-1/2}\beta_i
} So,  this like $\beta_i$ has a prior of $\normal(0, n \sig^2/\lm
A^{1/2})$. Note that the prior is only related to the estimated
variance of $\beta$, and $A$ appears with the power of $1/2$. Now
$A$ is not really the estimated variance of $\beta$, only the
variance of the estimates, hence it should be inflated, and the
square root takes care of that. Finally, note that eventually, if
$\beta_i$ is very large relative to $nA$, then the penalty become
$\|\beta\|$, so the ``prior'' becomes essentially normal, but with
exponential tails.

A  better way to look on the penalty from a Bayesian perspective
is to consider it as prior on the \(n\times p\) matrix
\(\scb=(\beta_1,\dots,\beta_n)\). Recall that the penalty is
invariant  to the rotation of the matrix \scb. In fact,
\(|||\scb|||_1 = |||T\scb U|||_1\), where \(T\) and \(U\) are
\(n\times n\) and \(p\times p\) rotation matrices. Now, this means
that if \(\be_1,\dots,\be_p\) are orthonormal set of eigenvectors
of \(\scb\t\scb\) and \(\gamma_{ij}=\be_j\t\beta_i\) --- the PCA
of \(\beta_1,\dots,\beta_n\), then \(|||\scb|||_1 = \summ j1p
\bigl(\summ i1n \gamma_{ij}^2\bigr)^{1/2} \) --- the RING lasso
penalty in terms of the principal components. The ``prior'' is
then proportional to $ e^{-\lm \summ j1p \|\gamma_{\cdot
j}\|_2}$.\, which is as if to obtain a random $\scb$ from the
prior the following procedure should be followed:

\begin{enumerate}
\item
Sample \(r_1,\dots,r_p\) independently from \(\Gamma(n,\lm)\) distribution.
\item
For each \(j=1,\dots,p\) sample \(\gamma_{1j},\dots,\gamma_{nj}\) independently and uniformly on the sphere with radius \(r_j\).
\item
Sample an orthonormal base \(\chi_1,\dots,\chi_p\) "uniformly''.
\item
Construct \(\beta_i = \summ j1p \gamma_{ik}\chi_k\).
\end{enumerate}


\subsection{Inequalities under an RE condition}

The assumption on the design matrix $X$ needs to be modified to
account for the search over rotations, in the following way.

\noindent{\bf Assumption RE2$(s, c_0, \kappa)$}. For some integer
$s$ such that $1 \leqslant s \leqslant p$, and a positive number
$c_0$ the following condition holds: \eqsplit{
 \kappa  =
  \min \{ & \frac{  ||X\t \Delta ||_2}{ \sqrt{m} ||P_V \Delta||_2 }: \, V
  \,\text{is a linear subspace of} \,\,
\mathbb{R}^p, \, \dim(V) \leqslant\, s, \,\\ & \Delta \in
\mathbb{R}^{p\times n}\setminus \{0\},
 |||(I-P_{V})\Delta |||_1  \leqslant \, c_0
 |||P_V \Delta|||_1   \} > 0,
} where $P_V$ is the projection on  linear subspace $V$.

If we restrict the subspaces $V$ to be of the form $V =
\bigoplus_{k=1}^r \langle e_{i_k}\rangle$, $r\leqslant s$ and
$\langle e_{i}\rangle$ is the linear subspace generated by the
standard basis vector $e_i$, and change the Schatten norm to
$\ell_{2,1}$ norm, then we obtain the restricted eigen value
assumption RE$_2(s,c_0,\kappa)$ of Lounici et
al.\ \cite{p:Lounici-GroupLasso}.

\begin{theorem}
\label{th:BRTspectral}

 Let \(y_{ij}\dist \normal (f_{ij},\sig^2)\) independent,
\(f_{ij}=x_{ij}\t\beta_i\), \(x_{ij}\in \R^p\),
\(\beta_i\in\R^p\), \(i=1,\dots,n\), \(j=1,\dots,m\), $p\geqslant
2$. Assume that $\summ j1m x_{ij\ell}^2 = m$ for all $i, \, \ell$.
 Let  assumption   RE2$(s, 3,\kappa )$ be satisfied for $X=(x_{ijl})$, where $s=\rank(\scb)$.
Consider the RING lasso estimator $\hat{f}_{ij} =
X_{ij}\t\hat{\beta}_i$ where $\hat{\scb}$ is defined by
\eqref{specPen} with $$\lm=4 \sigma \sqrt{(A+1)m np}, \quad
\text{for some}\quad A>1.$$

Then, for large $n$ or $p$, with probability at least \(1 -
 e^{-Anp/8}\),
 \eqsplit{
\frac 1 {mn} \|X\t(\scb - \hat{\scb})\|_2^2  & \leqslant
\frac{64(A+1)\sigma^2 s p}{\kappa^2\, m};\\
\frac 1 n ||| \scb-\hat\scb |||_1 & \leqslant \frac{32\sigma
\sqrt{1+A} \, s \,\sqrt{p} }{\kappa^2 \sqrt{mn}},\\
 \rank(\hat\scb) &\leqslant s\, \frac{64 \phi_{\max}
}{\kappa^2}, } where $\phi_{\max}$ is the maximal eigenvalue of
$X\t X/m$.
\end{theorem}

Thus we have bounds similar to those of group lasso as a function
of the threshold $\lambda$, with $s$ being the rank of $\scb$
rather than its sparsity. However, for RING lasso we need a larger
threshold compared to that of the group lasso ($\lm_{GL} = 4\sig
\sqrt{mn} \left(1+\frac{A\log p}{\sqrt{n}}\right)^{1/2}$, Lounici
et al.~\cite{p:Lounici-GroupLasso}).


\subsection{Persistence}

We discuss now the persistence of the RING lasso estimators (see
Section~\ref{sec:persistGen} for definition and a general result).

We  focus on the  sets which are related to the trace norm which
defines the RING lasso estimator:
$$B_{n,p} = \{\scb \in \mathbb{R}^{n\times p}: \, |||\scb|||_1 \leqslant b(n,p) \}.$$


\begin{theorem}\label{th:persist}

Assume that $n>1$. For any $ F\in \scf_{n,p}^m(V)$, $\beta\in
B_{n,p}$ and \eqsplit{\hat{\beta}^{(m,n,p)} = \argmin_{\beta\in
B_{n,p}} L_{\hat{F}}(\beta),} we have
 \eqsplit{
  L_{F}\left(\hat\beta\right) -
\min_{\beta\in B_{n,p}} L_{F}\left(\beta \right)   \leqslant
\left(\frac 1 m + \frac{p b^2}{nm}\right) \left(16e V \frac{
\log(np) }{ m \, \eta }\right)^{1/2} }
 with probability at least $1 - \eta$, for any $\eta\in (0,1)$.

\end{theorem}

\marg{I read slightly differently the theorem than you. This is
what I understand. Please check, accept or reject.}Thus, for
$\eta$ sufficiently small, the conditions $\log(np)\leqslant c_p
m^3 \eta $ and $b\leqslant c_b \sqrt{ {nm}/{p }}$, for some $c_b,
c_p >0$, imply that with sufficiently high probability, the
estimator is persistent. Roughly speaking, $b$ is the number of
components  in the SVD of \scb (the rank of \scb,  $\scm(\beta)$ after the proper rotation), and
if $m\gg\log n$, then what is needed is that this number will be
strictly less $ n^{1/2}m^{3/4}p^{-1/2} $. That is, if the true
model is sparse, $p$ can be almost as large as $m^{3/2}n^{1/2}$.

\marg{NB: I don't disagree but may be it is best to make it more
explicit that "roughly speaking" means treating the bound on
Schatten-$1$ norm as a bound on Schatten -0 norm?}


\subsection{Algorithm and small simulation study}

\twofigures[1.1]{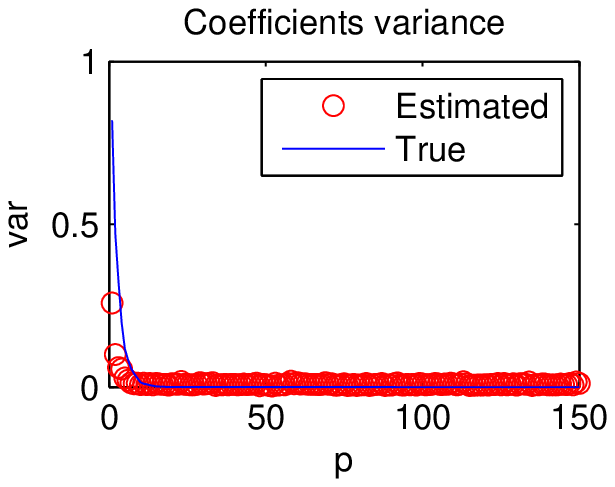}{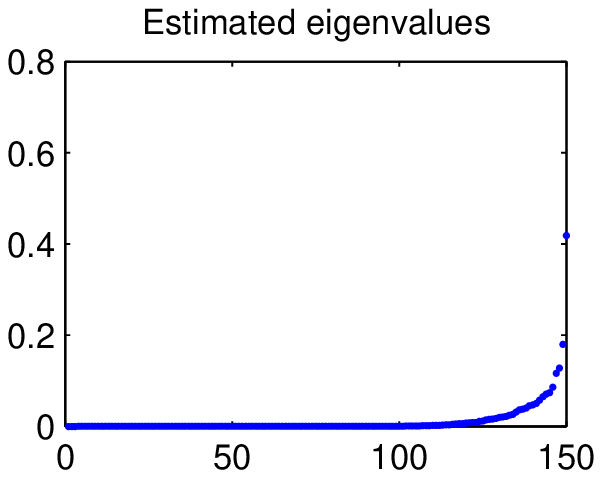}{Component
variances and eigenvalues, \(m=25\), \(n=150\) }{fig1}

A simple algorithm is the following:
\begin{enumerate}
\item Initiate some small value of \(\hat
\beta_1,\dots,\hat\beta_n\). Let \(A=\summ j1n
\hat\beta_j\hat\beta_j\t\). Fix  \(\gamma\in(0,1]\), \(\eps>0\),
\(k\), and \(c>1\). \item\label{A1st1} For \(i=1,\dots,n\):
\begin{enumerate}
\item Compute \(\delta_i = (X_i\t X_i + \lm A^{-1/2})^{-1}X_i\t
(y_i-X_i\hat\beta_i)\). \item Update \(A\leftarrow
A-\hat\beta_i\hat\beta_i\);
\(\hat\beta_i\leftarrow\hat\beta_i+\gamma\delta_i\); \(A\leftarrow
A+\hat\beta_i\hat\beta_i\);
\end{enumerate}
\item if \(\summ j1p \ind\bigl(n^{-1}\summ
i1n\hat\beta_{ij}^2>\eps\bigr)>k \) update
\(\lambda\leftarrow\lambda c\) otherwise
\(\lambda\leftarrow\lambda/ c\). \item Return to step \ref{A1st1}
unless there is no real change of coefficients.
\end{enumerate}


To fasten the computation, the SVD was computed only every 10
values of \(i\).

As a simulation we applied the above algorithm to the following
simulated data.
We generated random \(\beta_1,\dots,\beta_{150}\in\R^{150}\) such
that all coordinates are independent, and
\(\beta_{ij}\dist\normal(0, e^{-2j/5})\). All \(X_{ij\ell}\) are
\iid \(\normal(0,1)\), and \(y_{ij} = x_{ij\cdot}\t
\beta_i+\eps_{ij}\), where \(\eps_{ij}\) are all \iid
\(\normal(0,1)\).  The true \(R^2\) obtained was approximately
0.73. The number of replicates per value of \(\beta\), \(m\),
varied between 5 to 300. We consider two measures of estimation
error:
 \eqsplit{
    L_{\rm par} &= \frac{\summ i1n \|\hat\beta_i-\beta_i\|_\en} {\summ i1n \|\beta_i\|_\en}
    \\
    L_{\rm pre}  &= \frac{\summ i1n \|X_j(\hat\beta_i-\beta_i)\|_\en} {\summ i1n \|X_i\beta_i\|_\en}
  }

\onefigure[1.1]{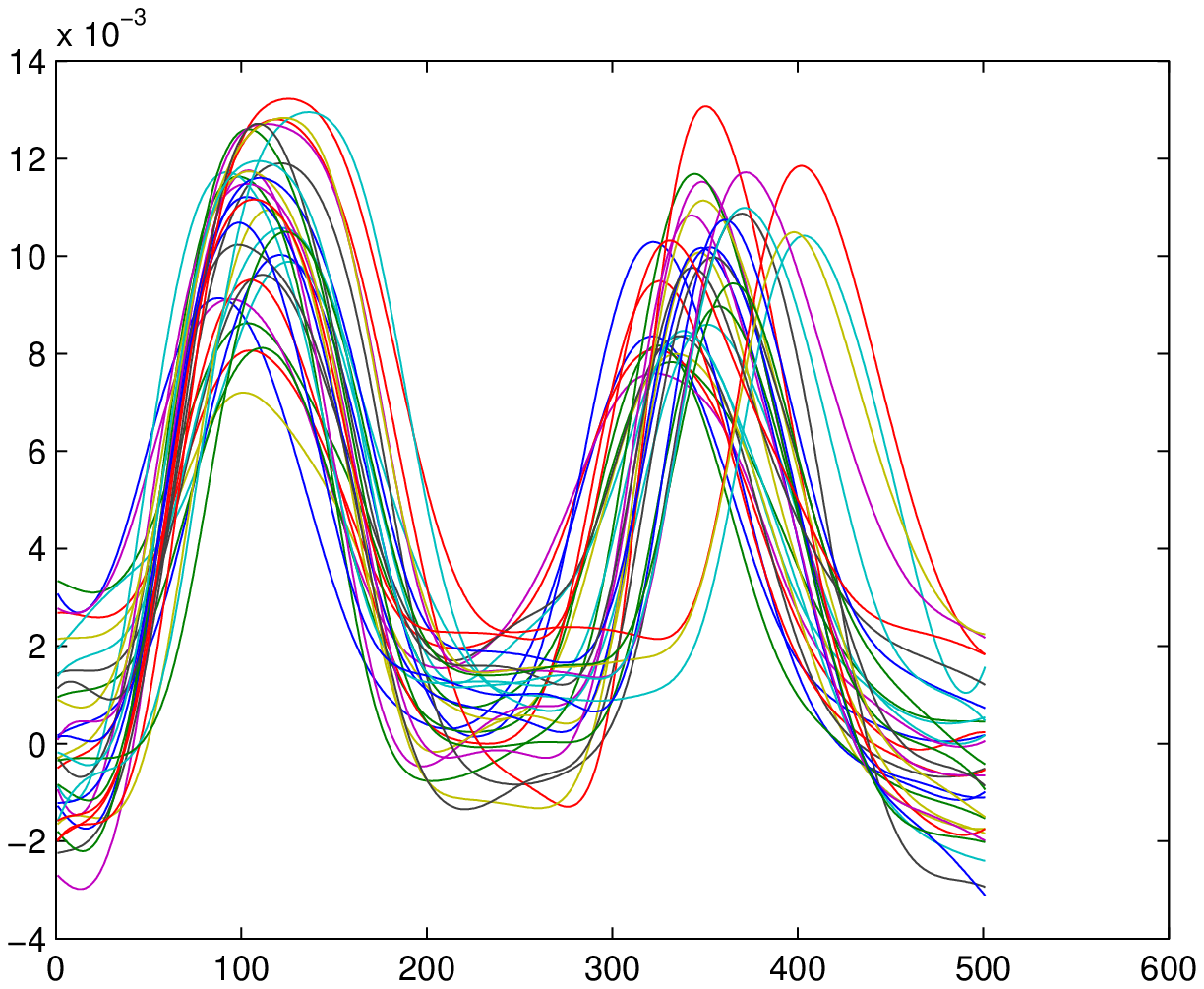}{ Lower lip position while
repeating 32 times 'Say bob again'}{fig2}

The algorithm stopped after 30--50 iterations. Figure \label{fig1}
is a graphical presentation of a typical result. A summary is
given in Table \ref{tab1}. Note that \(m\) has a critical  impact
on the estimation problem. However, with as little as \(5\)
observations per \(R^{150}\) vector of parameter we obtain a
significant reduction in the prediction error.
\begin{table}[H]
\caption{\label{tab1}The estimation and prediction error as
function of the number of observations per vector of parameters
Means (and SDK).}
\begin{center}

\begin{tabular}{|r|r|r|}
\hline
$m$     &   \(L_{\rm par}\)  & $L_{\rm pre}$    \\
\hline\hline
5  &    0.9530 (0.0075)               &   0.7349 (0.0375)         \\
\hline
25  &   0.7085 (0.0289)                &  0.7364 (0.0238)           \\
\hline
300  &  0.2470 (0.0080)                 &  0.5207 (0.0179)          \\
\hline

\end{tabular}
\end{center}
\end{table}

\threefiguresV[1.5]{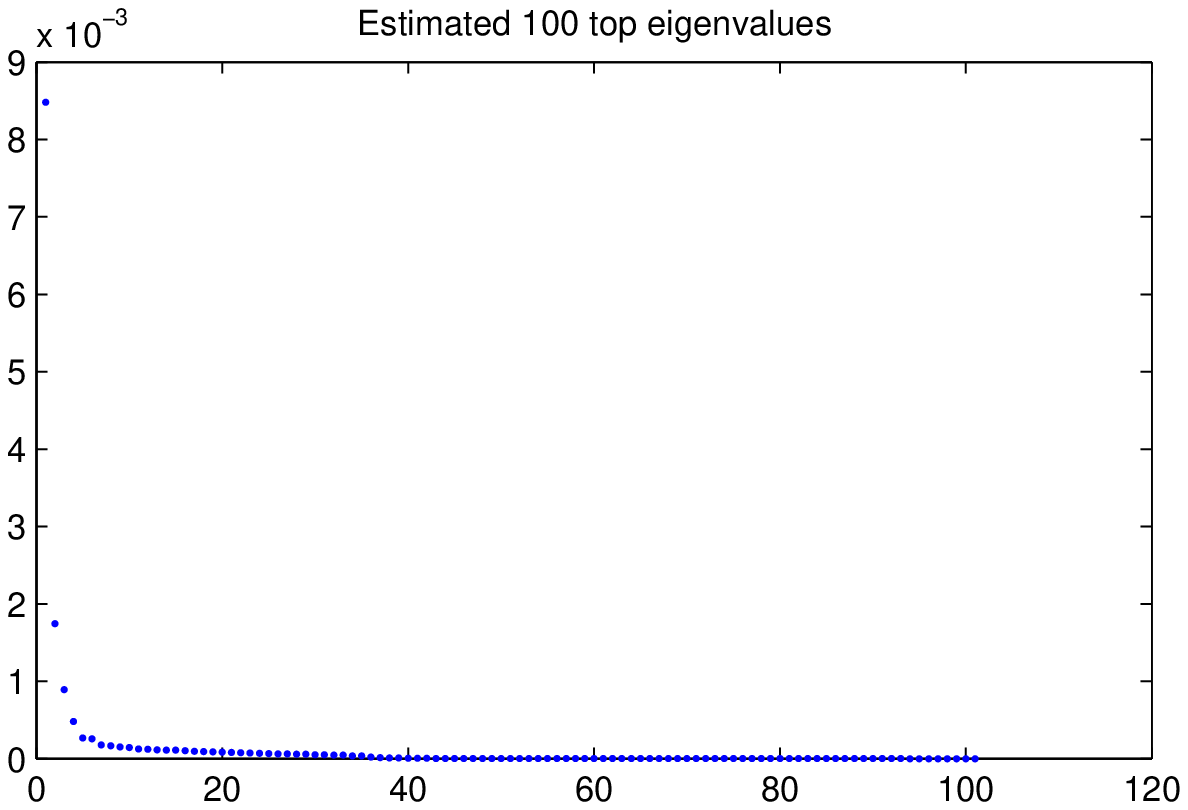}{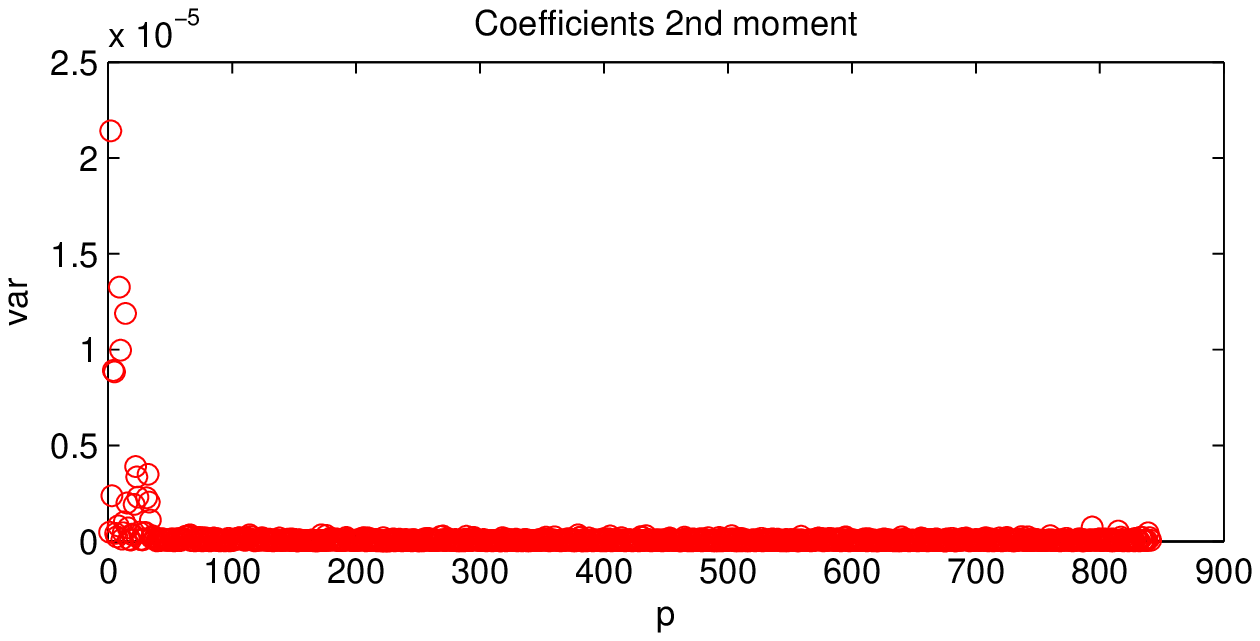}{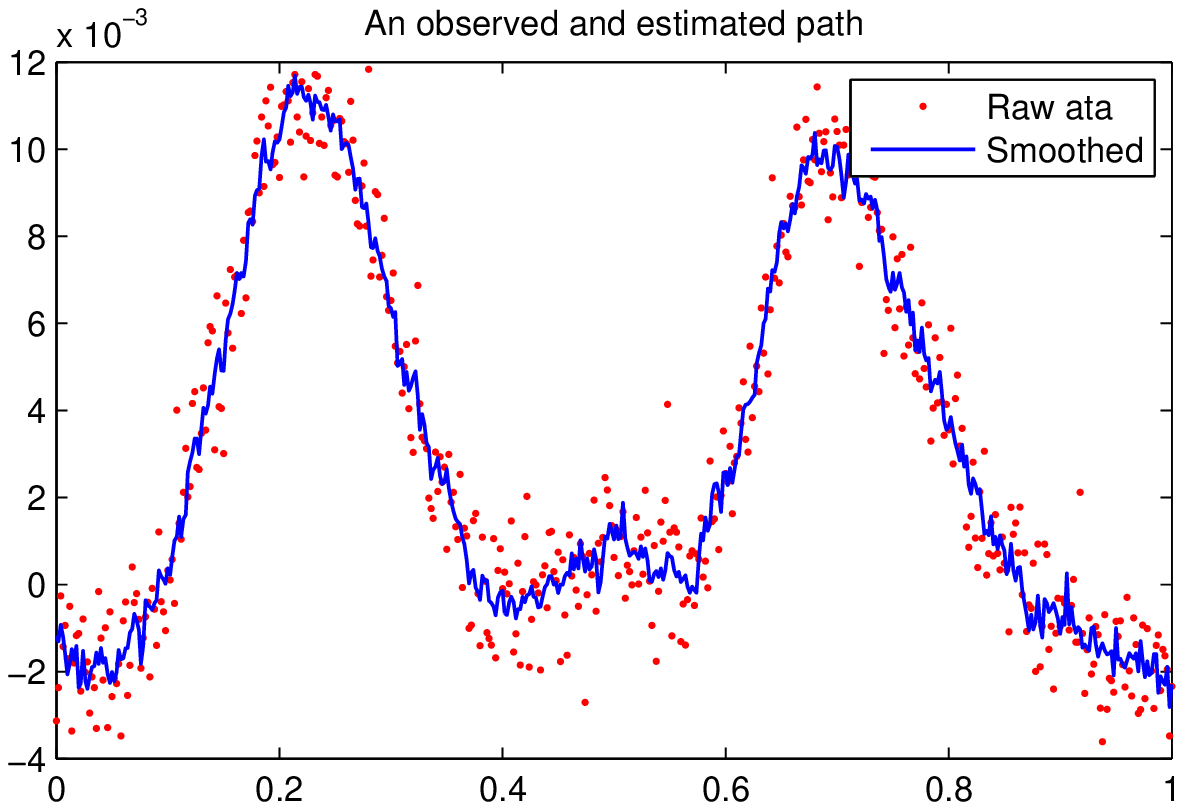}{Eigenvalue,
coefficient variance and typical observed and smooth path.}{fig3}

The technique is natural for functional data analysis. We used the
data LipPos. The data is   described by 
Ramsay and
Silverman and can be found in
http://www.stats.ox.ac.uk/~silverma/fdacasebook/lipemg.html. The
original data is given in Figure \ref{fig2}. However we added
noise to the data as can be seen in Figure \ref{fig3}. The lip
position is measured at $m=501$ time points, with \(n=32\)
repetitions.

As the matrix \(X\) we considered the union of 6 cubic spline
bases with, respectively, 5, 10, 20, 100, 200, and 500 knots
(i.e., \(p=841\), and \(X_i\) does not depend on \(i\)). A
Gaussian noise with \(\sigma=0.001\) was added to \(Y\). The
result of the analysis is given in Figure \ref{fig3}. Figure
\ref{fig4} presents the projection of the mean path on the first
eigen-vectors of \(\summ i1n \hat\beta_i\hat\beta_i\t\).

\onefigure[0.5]{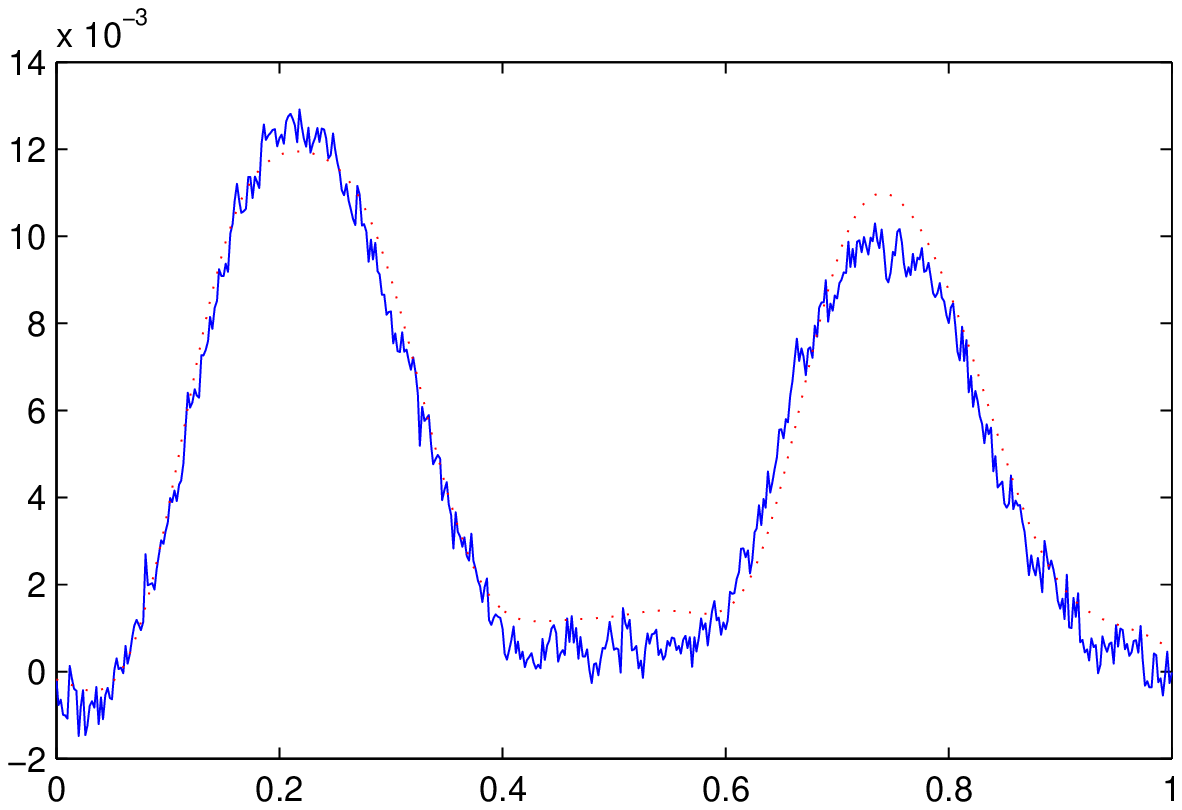}{Projection of the estimated mean
path on the 2 first eigen-vectors of \(\summ i1n
\hat\beta_i\hat\beta_i\t\) and the true mean path.}{fig4}

The final example we consider is somewhat arbitrary. The data,
taken from StatLib, is of the daily wind speeds for 1961-1978 at
12 synoptic meteorological stations in the Republic of Ireland. As
the \(Y\) variable we considered one of the stations (station
BIR). As explanatory variables we considered the 11 other station
of the same day, plus all 12 stations 70 days back (with the
constant we have altogether 852 explanatory variables). The
analysis was stratified by month. For simplicity, only the first
28 days of the month were taken, and the first year, 1961, served
only for explanatory purpose. The last year was served only for
testing purpose, so, the training set was for 16 years (\(n=12\),
\(m=448\), and \(p=852\) ).  In Figure \ref{fig5} we give the 2nd
moments of the coefficients and the scatter plot of predictions
vs. true value of the last year.

\twofigures[0.75]{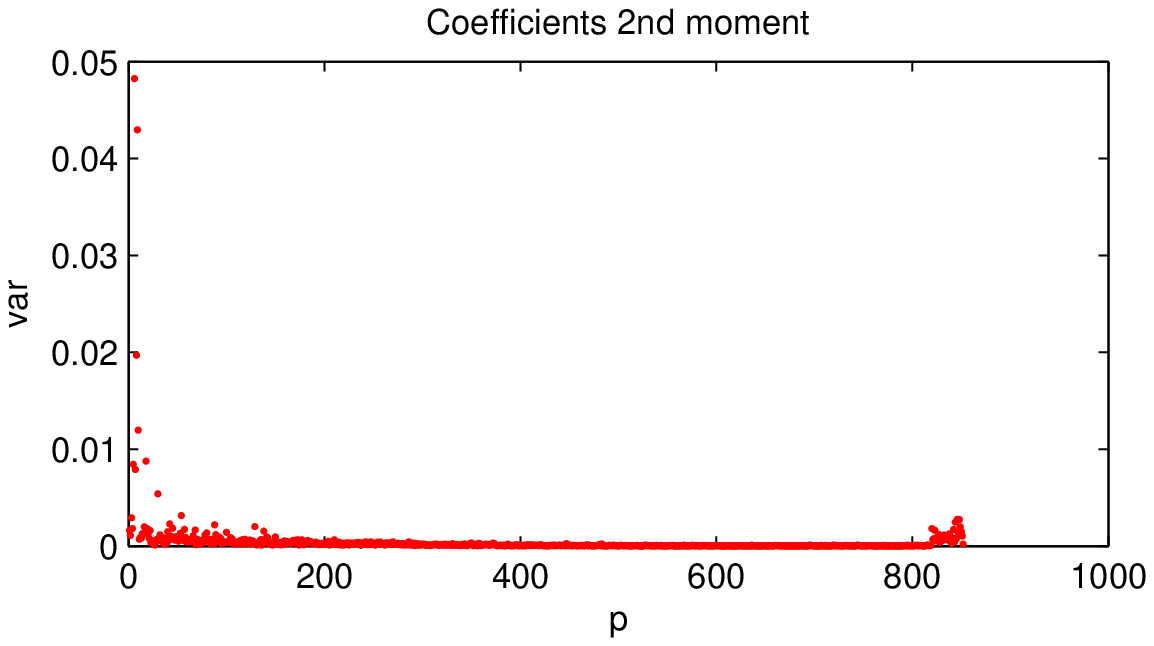}{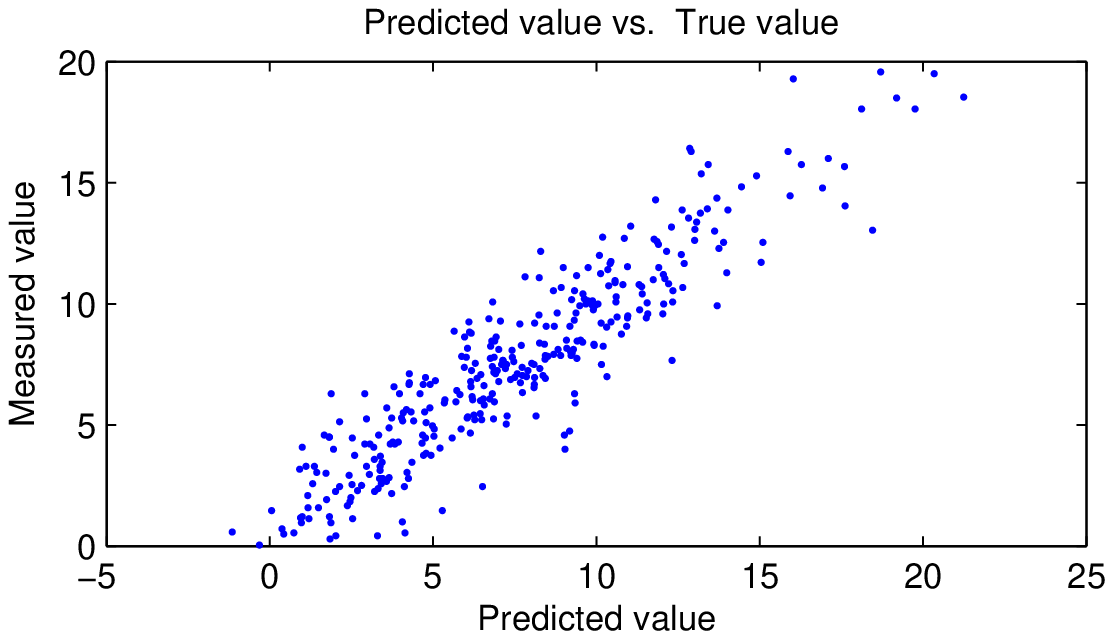}{Coefficient
2nd moment and prediction vs.true value of the test year.}{fig5}


\appendix

\section{Appendix}

\subsection{General persistence result.}\label{sec:persistGen}

A sequence of estimators $\hat{\beta}^{( m,n,p)}$ is persistent
with respect to a set of distributions  $\scf_{n,p}^m$ for $\beta
\in B_{n, p}$, if for any $F_{m,n,p}\in \scf_{n,p}^m$,
$$
L_{F_{m,n,p}}\left(\hat\beta^{(m,n,p)}\right) -
L_{F_{m,n,p}}\left(\beta^*_{F_{m, n,p}} \right)
\stackrel{P}{\rightarrow} 0,
$$
where $L_F(\beta) = (nm)^{-1}E_F \sum_{i=1}^n \summ j1m (Y_{ij} -
X_{ij}\t \beta_i)^2$, $F_{m, n, p}$ is the empirical distribution
function of $n\times (p+1)$ matrix $Z$,   $Z_i = (Y_i,
X_{i1},\dots, X_{ip})$, $i=1,\dots, n$, observed $m$ times. Here
$\beta^*_{F_{m, n,p}} = \argmin_{\beta \in B_{n,p}}
L_{F_{m,n,p}}(\beta)$, and $\scf_{n,p}^m$ stands for a collection
of distributions of $m$ observations of vectors $Z_i = (Y_i,
X_{i1},\dots, X_{ip})$, $i=1,\dots,n$.

\vspace{1ex}\noindent{\bf Assumption F}.  Under the distributions of random variables
$Z$  in $ \scf_{n,p}$,   $\xi_{i\ell k} = Z_{i\ell} Z_{i k}$
satisfy $E\left(\max_{i=1,\dots,n}\max_{\ell,k=1,\dots p+1}
\xi_{i\ell k}^2 \right) < V$. Denote this set of distributions by
$\scf_{n,p}(V)$.\par\vspace{1ex}

This assumption is similar to one of the assumptions of
Greenshtein and Ritov (2004). It is satisfied if, for instance,
the distribution of $Z_{i\ell}$ has finite support and the
variance of $Z_{i\ell} Z_{i k}$ is finite.

\begin{lemma}\label{lem:persist}
Let $ F\in \scf_{n,p}(V)$, and denote $\Sigma_i=(\sigma_{ijk})$
and $\hat\Sigma_{i}=(\hat{\sigma}_{ik\ell})$, with $\sigma_{ijk} =
E_F Z_{ij} Z_{ik}$ and $ \hat{\sigma}_{ik\ell} =  m^{-1}
\sum_{j=1}^m Z_{ik}^{(j)} Z_{i \ell}^{(j)}$, where $Z =
(Z_{i\ell}^{(j)})$ is a sample from $F^m$, $i=1,\dots,n$,
$j=1,\dots,m$, $\ell =1,\dots,p$.

Let $\hat{\beta}$ be the estimator minimising $\summ i1n \summ j1m
(Y_{ij} - X_{ij}\t\beta_i)^2$ subject to $\beta \in B$ where $B$
is some subset of $\mathbb{R}^{n\times p}$.

Then, for any $\eta\in(0,1)$, \eqsplit{ &\text{(a)\,}
\max_{i=1,\dots, n} ||\Sigma_i - \hat{\Sigma}_i||_{\infty}
&\leqslant&
\sqrt{\frac{2e V \log (n(p+1)^2)}{ m \, \eta}},\\
& \text{(b) \,} | L_{F}(\beta) - L_{\hat{F}}(\beta) | &\leqslant&
\frac 1 {nm } \sqrt{\frac{2e V   \log (n(p+1)^2)} {m\eta}} \left(
n+ \summ i1n \|\beta_i\|_1^2\right) } with probability at least $1
- \eta$.

\end{lemma}

\begin{proof}

Follows that of Theorem 1 in Greenshtein and Ritov (2004).

a)  Let $\hat\sigma_{ik\ell} =  \sigma_{ik\ell} +
\epsilon_{ik\ell}$, $E_i = (\epsilon_{ik\ell})$. Then, under
Assumption F and by Nemirovsky's inequality (see e.g. Lounici et
al~\cite{p:Lounici-GroupLasso}), \eqsplit{ &\hspace{-3em}P( \max_{i}||\Sigma_i -
\hat{\Sigma}_i||_{\infty} > A )\\ &\leqslant \frac 1 {A^2} E (\max_i
||\Sigma_i - \hat{\Sigma}_i||_{\infty}^2)\\ &\leqslant \frac {2e
\log (n(p+1)^2)}{mA^2} E(\max_{i=1,\dots,n}\max_{j,k=1,\dots p+1}
(Z_{ij} Z_{ik} - E(Z_{ij}
Z_{ik}))^2 )\\
& \leqslant \frac {2e V \log (n(p+1)^2) }{mA^2}. }

Taking $A = \sqrt{\frac{2e V   \log (n(p+1)^2)}{ m \, \eta}}$
proves the first part of the lemma.

b)  By the definition of $\hat{\beta}$ and $\beta^*_{F}$,
\eqsplit{ L_{F }(\hat{\beta} ) - L_{{F} }(\beta^*_{F }) \geqslant
0, \quad L_{\hat{F} }(\hat{\beta} ) - L_{\hat{F} }(\beta^*_{F })
\leqslant 0. }

Hence, \eqsplit{ 0 &\leqslant L_{F}\left(\hat\beta\right) -
L_{F}\left(\beta^*_{F} \right)  = L_{F}\left(\hat\beta\right) -
L_{\hat{F}}\left(\hat\beta\right)\\
&+L_{\hat{F}}\left(\hat\beta\right) -
 L_{{F}}\left(\hat\beta\right) +
 L_{{F}}\left(\hat\beta\right)
-L_{F}\left(\beta^*_{F} \right)\\
&\leqslant 2\sup_{\beta\in B_{n,p}}  |L_{F}(\beta) -
L_{\hat{F}}(\beta) |.}

Denote $\delta_i\t = (-1, \beta_{i,1}, \dots, \beta_{i,p})$, then
$$
L_F(\beta) = \frac 1 {nm}\summ i1n \delta_i\t \Sigma_{F,i}
\delta_i,
$$
where $\Sigma_{F,i}=(\sigma_{ijk})$ and $\sigma_{ijk} = E_F Z_{ij}
Z_{ik}$. For the empirical distribution function $\hat{F}_{mn}$
determined by a sample $Z_{i \ell}^{(j)}$, $i=1,\dots,n$,
$j=1,\dots,m$, $\ell =1,\dots,p$,
$\Sigma_{\hat{F},i}=(\hat{\sigma}_{ik\ell})$ and $
\hat{\sigma}_{ik\ell} = \frac 1 m \sum_{j=1}^m Z_{i k}^{(j)}
Z_{i\ell}^{(j)}. $

Introduce matrix $\hat{\sce} $ with $\hat{\sce}_{j\ell} = A$. Hence,
with probability at least $1-\eta$,

\eqsplit{ | L_{F}(\beta) - L_{\hat{F}}(\beta) | &= \left|\frac 1
{nm}\summ i1n  \delta_i \t (\Sigma_{F,i} - \Sigma_{\hat{F}}, i)
\delta_i \right|
\\& \leqslant \frac 1 {nm}\summ i1n |\delta_i|\t
\hat{\sce} |\delta_i|\\ &= \frac 1 {nm } \sqrt{\frac{2e V \log
(n(p+1)^2)} {m\eta}} (n+ \summ i1n \|\beta_i\|_1^2). }

\end{proof}

\subsection{Proofs of Section \ref{sec:lassoes}}

\begin{proof}[Proof of Theorem \ref{th:lassoes1}]

Note that by the definition of $\tihat\beta_i$ and \eqref{roughapprox}.
 \eqsplit[et1]{
    &\hspace{-3em}mnc_n +\lm_n\summ i1n\|\tihat\beta_{i}\|_1^2
    \\
    &\leq
    m\summ i1n \tihat\beta_i\t\ti \Sig_i \tihat\beta_i+\lm_n\summ i1n \|\tihat\beta_i\|_1^2
    \\
    &\leq m\summ i1n \tihat\beta_i\t\ti S_i \tihat\beta_i+(\lm_n+m\del_n)\summ i1n \|\tihat\beta_i\|_1^2
    \\
    &\leq   m\summ i1n \ti\beta_{i0}\t\ti S_i \ti\beta_{i0}+\lm_n\summ i1n \|\ti\beta_{i0}\|_1^2
    +m\del_n\summ i1n \|\tihat\beta_i\|_1^2
    \\
    &\leq   m\summ i1n \ti\beta_{i0}\t\ti \Sig_i \ti\beta_{i0}
    + (\lm_n+m\del_n)\summ i1n \|\ti\beta_{i0}\|_1^2
    + m\del_n\summ i1n \|\tihat\beta_i\|_1^2
    \\
    &= mnC_n + (\lm_n+m\del_n)\summ i1n \|\ti\beta_{i0}\|_1^2
    + m\del_n\summ i1n \|\tihat\beta_i\|_1^2  .
  }
Comparing the LHS with the RHS of \eqref{et1}, noting that $m\del_n\ll\lm_n$:
 \eqsplit{
    \summ i1n \|\tihat\beta_i\|_1^2
    &\leq mn\frac{C_n-c_n}{\lm_n-m\del_n}
    + \frac{\lm_n+m\del_n}{\lm_n-m\del_n} \summ i1n \|\ti\beta_{i0}\|_1^2.
    }
By  \eqref{roughapprox} and \eqref{lassoes}:
 \eqsplit[bpl1]{
    \summ i1n \tihat\beta_{i}\t\ti \Sig_i \tihat\beta_{i}
    &\leq \summ i1n \tihat\beta_{i}\t\ti S_i \tihat\beta_{i} + \del_n\summ i1n \|\tihat\beta_{i}\|_1^2
    \\
    &\leq   \summ i1n \ti\beta_{i0}\t\ti S_i \ti\beta_{i0}
    + \frac {\lm_n} m \summ i1n \|\ti\beta_{i0}\|_1^2
    - \frac {\lm_n} m \summ i1n \|\tihat\beta_{i}\|_1^2 +  \del_n\summ i1n \|\tihat\beta_{i}\|_1^2
    \\
    &\leq   \summ i1n \ti\beta_{i0}\t\ti \Sig_i \ti\beta_{i0}
    + (\frac{\lm_n} m+\del_n) \summ i1n \|\ti\beta_{i0}\|_1^2
    - (\frac {\lm_n} m-\del_n) \summ i1n \|\tihat\beta_{i}\|_1^2
    \\
    &\leq   \summ i1n \ti\beta_{i0}\t\ti \Sig_i \ti\beta_{i0}
    + (\frac{\lm_n} m+\del_n) \summ i1n \|\ti\beta_{i0}\|_1^2
.
  }
 The result follows.
\end{proof}

\begin{proof}[Proof of Theorem \ref{th:lassoes2}]
The proof is similar to the proof of Theorem \ref{th:lassoes1}. Similar to \eqref{et1} we obtain:
 \eqsplit[et2]{
    &\hspace{-1em}mnc_n +\lm_n\summ i1n\|\tihat\beta_{i}\|_1^\al
    \\
    &\leq
    m\summ i1n \tihat\beta_i\t\ti \Sig_i \tihat\beta_i+\lm_n\summ i1n \|\tihat\beta_i\|_1^\al
    \\
    &\leq m\summ i1n \tihat\beta_i\t\ti S_i \tihat\beta_i
    +\lm_n\summ i1n \|\tihat\beta_i\|_1^\al
    +m\del_n\summ i1n \|\tihat\beta_i\|_1^2
    \\
    &\leq   m\summ i1n \ti\beta_{i0}\t\ti S_i \ti\beta_{i0} + \lm_n\summ i1n \|\ti\beta_{i0}\|_1^\al
    +m\del_n\summ i1n \|\tihat\beta_i\|_1^2
    \\
    &\leq   m\summ i1n \ti\beta_{i0}\t\ti \Sig_i \ti\beta_{i0}
    + \lm_n\summ i1n \|\ti\beta_{i0}\|_1^\al
    + m\del_n\summ i1n \|\ti\beta_{i0}\|_1^2
    + m\del_n\summ i1n \|\tihat\beta_i\|_1^2
    \\
    &= mnc_n + \lm_n\summ i1n \|\ti\beta_{i0}\|_1^\al
    + m\del_n\summ i1n \|\ti\beta_{i0}\|_1^2
    + m\del_n\summ i1n \|\tihat\beta_i\|_1^2 .
  }
That is,
 \eqsplit[alsqb]{
    \summ i1n (\lm_n \|\tihat\beta_i\|_1^\al - m\del_n \|\tihat\beta_i\|_1^2)
    &\leq  \lm_n\summ i1n \|\ti\beta_{i0}\|_1^\al
    + m\del_n\summ i1n \|\ti\beta_{i0}\|_1^2
    \\
    &= \O(mn\del_n).
  }
It is easy to see that the maximum  of $\summ
i1n\|\tihat\beta_i\|_1^2$ subject to the constraint \eqref{alsqb}
is achieved when $\|\tihat\beta_1\|_1^2 = \dots =
\|\tihat\beta_n\|_1^2$. That is when $\|\tihat\beta_i\|_1^2$
solves $\lm_n u^\al - m\del_n u^2 = \O(m\del_n)$. As $\lambda_n
=\O(m \delta_m) $, the solution satisfies
$u=\O(m\del_n/\lm_n)^{1/(\al-2)}$.

Hence we can conclude from \eqref{alsqb}
 \eqsplit{
    \summ i1n \|\tihat\beta_i\|_2^2 &= \O(n (m\del_n/\lm_n)^{2/(\al-2)} )
  }
We now proceed similar to \eqref{bpl1}
 \eqsplit{
    \summ i1n \tihat\beta_{i}\t\ti \Sig_i \tihat\beta_{i}
    &\leq \summ i1n \tihat\beta_{i}\t\ti S_i \ti\beta_{i} + \del_n\summ i1n \|\tihat\beta_{i}\|_1^2
    \\
    &\leq   \summ i1n \ti\beta_{i0}\t\ti S_i \ti\beta_{i0}
    + \frac{\lm_n}m  \summ i1n \|\ti\beta_{i0}\|_1^\al
    -  \frac{\lm_n}m  \summ i1n \|\tihat\beta_{i}\|_1^\al
    +  \del_n\summ i1n \|\tihat\beta_{i}\|_1^2
    \\
    &\leq   \summ i1n \ti\beta_{i0}\t\ti \Sig_i \ti\beta_{i0}
    + \frac{\lm_n}m \summ i1n \|\ti\beta_{i0}\|_1^\al
    +  \del_n \summ i1n \|\ti\beta_{i0}\|_1^2
    + \del_n \summ i1n \|\tihat\beta_i\|_1^2
    \\
    &\leq   \summ i1n \ti\beta_{i0}\t\ti \Sig_i \ti\beta_{i0}
    + \O_p(n (m/\lm_n)^{2/(\al-2)} \del_n^{\al/(\al-2)}),
  }
since $\lambda_n =\O(m \delta_m) $.

\end{proof}

\begin{proof}[Proof of Remark~\ref{rem:lassoes2}]

If  $m \delta_m/\lambda =\o(1)$, then, following the proof of
Theorem~\ref{th:lassoes2}, the solution maximising $\summ
i1n\|\tihat\beta_i\|_1^2$ subject to the constraint \eqref{alsqb}
satisfies $\|\tihat\beta_i\|_1=\O(1)$, and hence we have
 \eqsplit{
    \summ i1n \tihat\beta_{i}\t\ti \Sig_i \tihat\beta_{i}
    &\leq  \summ i1n \ti\beta_{i0}\t\ti \Sig_i \ti\beta_{i0}
    + \O_p\left( n\lm_n/m  + n\del_n\right).
  }

\end{proof}


\begin{proof}[Proof of Theorem \ref{th:LassoL1p}]

The proof follows that of Lemma 3.1 in Lounici et
al.~\cite{p:Lounici-GroupLasso}.

We start with (a) and (b). Since \(\hat\beta\) minimizes \eqref{lassoes}, then,
\(\forall\beta\)
 \eqsplit{
    \summ i1n ||Y_i -X_i\t\hat\beta_i||_2^2+\lm
    \summ i1n \|\hat\beta_i\|_1^{\alpha}
    &\leq
    \summ i1n ||Y_i -X_i\t\beta_i||_2^2+\lm \summ i1n \|\beta_i\|_1^{\alpha},
  }
and hence, for $Y_i = X_i\t \beta_i + \varepsilon_i$,
 \eqsplit{
 \summ i1n ||X_i\t(\hat\beta_i - \beta_i)||_2^2 \leqslant  \summ i1n\left[ 2\varepsilon_i\t
X_i\t (\beta_i - \hat\beta_i) + \lambda (||\beta_i||_1^{\alpha} -
||\hat\beta_i||_1^{\alpha}) \right].  }


Denote \(V_{i\ell}=\summ j1m x_{ij\ell}\eps_{ij}\dist\normal(0,m
\sig^2 )\), and  introduce event \(\sca_i =\bigcap_{\ell=1}^p
\{|V_{i\ell}|\leq \mu\}\), for some \(\mu>0\). Then
 \eqsplit{
    P(\sca_i^c) &\leq \summ \ell 1p P(|V_{i\ell}|>\mu)
    \\
    &= \summ \ell 1p 2\biggl[1-\Phi\Bigl\{\mu/(\sig\sqrt{m})\Bigr\}\biggr]
    \\
    &\leq  p \exp\bigl\{- \mu^2/(2m\sig^2)\bigr\}.
   }
For $\sca=\cap_{i=1}^n \sca_i$, due to independence,
$$P(\sca^c) =
\summ i1n P(\sca_i^c) \leqslant  pn \exp\bigl\{-
\mu^2/(2m\sig^2)\bigr\}.
$$

Thus, if \(\mu\) is large enough, \(P(\sca^c)\) is small, e.g.,
for \(\mu=\sig A\bigl(m \log (np)\bigr)^{1/2}\), \(A>\sqrt 2\), we
have \(P(\sca^c)\leq (np)^{1-A^2/2}\).

On event $\sca$, for some $\nu>0$,
 \eqsplit { &\hspace{-1em} \summ i1n \left[||X_i(\hat\beta_i - \beta_i)||_2^2 + \nu ||\beta_i -
\hat\beta_i||_1  \right]
\\
&\leqslant \summ i1n \left[
  2\mu   ||\beta_i -
\hat\beta_i||_1 + \lambda (||\beta_i||_1^2 -
||\hat\beta_i||_1^2)\right.
+ \left. \nu ||\beta_i - \hat\beta_i||_1 \right]
\\
&=   \summ i1n\summ j1m \left[\al \, \lambda
\max(||\beta_i||_1^{\alpha-1}, ||\hat\beta_i||_1^{\alpha-1}) (|
\beta_{ij}| - |\hat\beta_{ij}|) +
(\nu+2\mu)|\beta_{ij} - \hat\beta_{ij}| \right]\\
&\leqslant    \summ i1n\summ j1m \left[ \al \, \lambda
\max(B^{\alpha-1}, \hat{B}^{\alpha-1})(| \beta_{ij}|  -
|\hat\beta_{ij}|) + (\nu+2\mu)|\beta_{ij} - \hat\beta_{ij}|
\right],
 }
due to inequality $|x^\alpha - y^\alpha| \leqslant \alpha |x-y|
\max(|x|^{\alpha-1},  |y|^{\alpha-1})$ which holds for $\alpha
\geqslant 1$ and any $x$ and $y$. To simplify the notation, denote
$\scc =\al \, \max(B^{\alpha-1}, \hat{B}^{\alpha-1})$.
%

Denote $J_i = J(\beta_i) = \{j:
\,\, \beta_{ij} \neq 0 \}$, $\scm(\beta_i) = |J(\beta_i)|$. For
each $i$ and $j\in J(\beta_i)$, the expression in square brackets
is bounded above by
$$
[\lambda \scc   +  \nu+2\mu] \,|\beta_{ij} - \hat\beta_{ij}|,
$$
and for  $j\in J^c(\beta)$, the expression in square brackets is
bounded above by $0$, as long as $\nu+2\mu \leqslant \lambda
\scc$:
$$
- \lambda \scc |\hat\beta_{ij}| +
 (\nu+2\mu)|\hat\beta_{ij}| \leqslant 0.
$$
This condition is satisfied if $\nu+2\mu \leqslant \lambda \scc$.

Hence, on $\sca$, for  $\nu+2\mu \leqslant \lambda \scc$,
 \eqsplit{
\summ i1n  \left[ ||X_i\t(\hat\beta_i - \beta_i)||_2^2 + \nu
||\beta_i - \hat\beta_i||_1 \right] \leqslant \summ i1n  [\lambda
\scc +   2\mu + \nu] ||(\beta_i - \hat\beta_i)_{J_i}||_1.
 }

 This implies that
 \eqsplit{ \summ i1n ||X_i(\hat\beta_i - \beta_i)||_2^2 \leqslant
    [\lambda \scc +  \nu + 2\mu] ||(\beta  -
\hat\beta )_J||_1,
 }
as well as
that
 \eqsplit {
 ||\beta - \hat\beta||_1 \leqslant
 \left[ 1 + \frac{2\mu}{\nu} + \frac{\lambda}{\nu}
\scc  \right] ||(\beta  - \hat\beta )_J||_1.
 }

Take $\nu =\lambda \scc/2$, hence we need to assume that $ 2\mu
\leqslant \lambda \scc/2$:
 \eqsplit[eq:boundXdelta]{
 \summ i1n ||X_i\t(\hat\beta_i - \beta_i)||_2^2 &\leqslant
    \left[ \frac {3\lambda}{2} \scc + 2\mu\right] ||(\beta  -
\hat\beta )_J||_1,\\
 ||\beta - \hat\beta||_1 &\leqslant
 \left[ 3 + \frac{4\mu} {\lambda \scc}  \right] ||(\beta  -
\hat\beta )_J||_1 \leqslant  4  ||(\beta - \hat\beta)_J||_1.
 }
which implies
 \eqsplit {
 ||(\beta - \hat\beta)_{J^c}||_1 \leqslant 3 ||(\beta -
\hat\beta)_J||_1.
 }

Due to the generalized restricted eigenvalue assumption RE$_1(s,
3, \kappa)$, $||X\t(\beta-\hat\beta)||_2 \geqslant \kappa \sqrt{m}
||(\beta -
 \hat\beta)_J||_2$, and hence, using  \eqref{eq:boundXdelta},
 \eqsplit { ||X\t(\hat\beta - \beta)||_2^2  &\leqslant \left[ \frac {3\lambda}{2}\scc + 2\mu\right]
  \sqrt{n\scm(\beta)}  || (\hat\beta -
\beta)_J||_2\\&\leqslant \left[ \frac {3\lambda}{2} \scc +
2\mu\right] \frac{ \sqrt{n\scm(\beta)}}{\kappa \sqrt{m} }
||X\t(\hat\beta - \beta)||_2,
 }
where $\scm(\beta)=\max_i \scm(\beta_i)$, implying that
 \eqsplit { ||X\t(\hat\beta - \beta)||_2
&\leqslant \left[ \frac {3\lambda}{2} \scc + 2\mu\right] \frac{
\sqrt{n \scm(\beta)}}{\kappa \sqrt{m} }\\ &= \frac{ \sqrt{n
\scm(\beta)}}{\kappa \sqrt{m} } \left[ \frac {3\lambda}{2} \scc +
2A \sig \sqrt{m\log(np)}\right].
 }

 Also, \eqsplit{ ||\beta - \hat\beta||_1 &\leqslant 4||(\beta -
\hat\beta)_J||_1  \leqslant 4\frac{\sqrt{n\scm(\beta)}}{\sqrt{m}
\kappa} ||X\t(\beta - \hat\beta)||_2\\ &\leqslant
 \frac{4n \scm(\beta)}{m \kappa^2} \left[ \frac {3\lambda}{2} \scc + 2A \sig \sqrt{m\log(np)}\right].
 }
Hence, a) and b) of the theorem are proved.


(c) For $i$, $\ell$: $\hat\beta_{i\ell}\neq 0$, we have
 \eqsplit{
   2 X_{i\cdot \ell} (Y_i  -X_i\t\hat\beta_i) &=\lm
    \alpha  \sgn(\hat\beta_{i\ell}) ||\hat\beta_i||_1^{\alpha-1}\, ,
  }

Hence,
 \eqsplit{
  \sum_{\ell: \, \hat\beta_{i\ell}\neq 0} ||X_{i\cdot \ell} X_i\t(\beta_i-\hat\beta_i)||_2^2 &\geqslant
 \sum_{\ell: \, \hat\beta_{i\ell}\neq 0} \left(  ||X_{i\cdot \ell}(Y_i
    -X_i\t\hat\beta_i)||_2 -  ||X_{i\cdot \ell}(Y_i
    -X_i\t\beta_i)||_2 \right)^2\\
    &\geq
    \sum_{\ell:\hat\beta_{i\ell}\ne0}\Bigl(\al \,\lm  ||\hat\beta_i||_1^{\al - 1}/2  -  \mu
    \Bigr)^2
    \\
    &=\scm(\hat\beta_i) (\al \,\lm  ||\hat\beta_i||_1^{\al - 1}/2  -  \mu)^2.
  }

Thus, 
 \eqsplit{
    \scm(\hat\beta_i) &\leq \|X_i(\beta_i-\hat\beta_i)\|_2^2
    \frac{m  \phi_{i,\, \max}} {\left(\lm  \alpha  ||\hat\beta_i||_1^{\alpha-1}/2-\mu\right)^2 }.
  }

Theorem is proved.
\end{proof}


\begin{proof}[Proof of Theorem \ref{th:L12merge2}.]

To satisfy the conditions of Theorem~\ref{th:LassoL1p}, we can
take $B = b$ and $\lm = \frac{4A\sig}{\al b^{\alpha-1}}\sqrt{m
\log(np)}$.

Thus, by Lemma~\ref{lem:persist},
$$
\frac{\lambda}{m\delta_n} = \frac{4A\sig}{\al b^{\al
-1}}\sqrt{\frac{\log(np)}{m}} \sqrt{\frac{ m \, \eta}{2e V \log
(n(p+1)^2)}} = C \frac{\sqrt{\eta}}{\al b^{\al -1}}\leqslant C_1,
$$
hence assumption $\lambda=\O(m\delta_n)$ of
Theorem~\ref{th:lassoes2} is
satisfied.

Hence, from the proof of  Theorem~\ref{th:LassoL1p}, it follows
that \eqsplit{
 \|\hat\beta_i\|_1 &= \O \left( (m\del_n/\lm_n)^{1/(\al-2)} \right) = \O \left( \left(\frac{b^{\al -1}}{\sqrt{\eta}}\right)^{1/(\al-2)} \right). }

Hence, we can take $B = b$ and $\hat{B}= C\left(\frac{b^{\al
-1}}{\sqrt{\eta}}\right)^{1/(\al-2)} $ for some $C>0$, and apply
Theorem~\ref{th:LassoL1p}. Then $\max(1, \hat{B}/B)$ is bounded by
$$
\max\left[1, C\frac{ b^{(\al-1)/(\al -2)-1} }{\eta^{1/(2(\al-2))}
}\right] =
  \max\left[1, C\frac{ b^{1/(\al -2)} }{\eta^{1/(2(\al-2))} }
 \right] =  \left(\frac{C b }{\sqrt{\eta}} \right)^{1/(\al-2)},
$$
since $\frac{C b }{\sqrt{\eta}} \geqslant C_2
\frac{\eta^{1/(2(\al-1))} }{\sqrt{\eta}} \geqslant C_2
\eta^{-(\al-2)/(2(\al-1))}$ is large for small $\eta$.

Hence, \eqsplit{ &\frac {3 \alpha \lambda}{2\sqrt{ m}}   \max(
B^{\alpha-1}, \hat{B}^{\alpha-1}) + 2A \sig \sqrt{ \log(np) }\\
&\leqslant  6A C \sig \sqrt{ \log(np) }   \frac{ b^{(\al-1)/(\al
-2)} }{\eta^{(\al-1)/(2(\al-2))} } + 2A \sig \sqrt{ \log(np) }\\
&= 2A \sig \sqrt{ \log(np) } \left[3 C \left(\frac{ b
}{\sqrt{\eta}}\right)^{(\al-1)/(\al -2)}
 + 1 \right], } and, applying Theorem~\ref{th:LassoL1p}, we
obtain (a) and (b).

c) Apply c) in Theorem~\ref{th:LassoL1p}, summing over $i\in
\sci$:
 \eqsplit{
    \sum_{i\in \sci} \scm(\hat\beta_i) &\leq \|X \t(\beta -\hat\beta )\|_2^2
    \frac{m \phi_{ \max}} {(\mu\delta)^2}\\
    &\leq \frac{4   s n \phi_{ \max} }{\kappa^2 \, \delta^2 }
    \left[1 + 3C \left(\frac{ b
}{\sqrt{\eta}}\right)^{(\al-1)/(\al -2)}
 \right]^2. }

\end{proof}

\subsection{Proofs of Section \ref{sec:spectral} }

\begin{proof}[Proof of Lemma \ref{lem:spectVsGroup}]
Let $\scb=\summ \xi1k \al_\xi \beta_\xi^* {\be_\xi^*}\t$ be the
spectral decomposition of $\scb$, where
$\beta_1^*,\dots,\beta_k^*$ are orthonormal $\R^p$ vectors,
$\be_1^*,\dots,\be_k^*$ are orthonormal $\R^n$ vectors,
$\al_1,\dots,\al_k\ge 0$, and $k=\min\{p,n\}$.  Clearly
$|||\scb|||_1=\summ \xi1k \al_\xi$. Let  $U=\summ \xi1k
e_\xi{\beta_\xi^*}\t $ where $e_1,\dots,e_p$ is the natural basis
of $\R^n$. Then
 \eqsplit{
    \|U\scb\|_{2,1} &=  \| \summ\xi1k \al_\xi e_\xi{\be_\xi^*}\t \|_{2,1} = \summ\xi1k \al_\xi = |||\scb|||_1.
  }
Let $\scb=\summ \xi1k e_\xi\be_\xi\t$ where $\be_1,\be_2,\dots,\be_k$ are orthogonal, and let $U$ be a unitary matrix. Then by Schwarz inequality
 \eqsplit{
    \|\scb\|_{2,1}&=
    \summ j1p \|\be_j\|
    \\
    &=  \summ i1p \summ j1p U_{ij}^2 \|\be_j\|\qquad&\text{since} \summ i1pU_{ij}^2=1
    \\
    &\le \summ i1p \sqrt {\summ j1p U_{ij}^2  \|\be_j\|^2}\sqrt{\summ j1p U_{ij}^2 }\qquad&\text{by Schwarz inequality}
    \\
    &= \summ i1p \sqrt {\summ j1p U_{ij}^2  \|\be_j\|^2} \qquad&\text{since} \summ j1pU_{ij}^2=1
    \\
    &= \|U\scb\|_{2,1}
  }
which completes the proof of the (i).

Now, consider the $U$ defined as above for the solution of
\eqref{specPen}. Let $\ti X_i$ be the design matrices $\ti\scb$ be
the solution expressed in this basis. By the first part of the
lemma $|||\ti\scb|||_1=\|\ti\scb\|_{2,1}$.  Suppose there is a
matrix $\scb\ne\ti\scb$ which minimizes the group lasso penalty.
Hence
 \eqsplit{
    \summ i1n \|Y_i-\ti X_i  \beta_i\|^2 + \lm ||| \scb|||_1
    &\leq
    \summ i1n \|Y_i-\ti X_i \beta_i\|^2 + \lm \|\scb\|_{2,1}
    \\
    &<  \summ i1n \|Y_i-\ti X_i \ti \beta_i\|^2 + \lm \|\ti\scb\|_{2,1}
    \\
    &= \summ i1n \|Y_i-\ti X_i  \ti\beta_i\|^2 + \lm |||\ti \scb|||_1,
  }
contradiction since $\ti\scb$ minimized \eqref{specPen}. Part (ii) is proved.
\end{proof}

\begin{proof}[Proof of Theorem \ref{th:sparseSpect} ]

Let $A=\summ i1n \hat\beta_i\hat\beta_i\t = \hat\scb \hat\scb\t$
be of rank $s\leqslant p < n$, and hence the spectral
decomposition of $\hat\scb$ can be written as $\hat\scb=\summ
\xi1s \al_\xi \beta_\xi^{*}{\be_\xi^*}\t$, where
$\beta_1^*,\dots,\beta_s^* \in \mathbb{R}^p$ are orthonormal, and
so are $\be_1^*,\dots,\be_s^* \in \mathbb{R}^n$. Hence, the
rotation $U$ leading to a sparse representation $U \hat\scb$ (with
$s$ non-zero rows) is given by $U =\summ \xi1s e_\xi
{\beta_\xi^*}\t$, where $e_1,\dots,e_p$ is the natural basis of
$\R^p$. Another way to write the rotation matrix is $U =
({\beta_1^*}\t,\dots, {\beta_s^*}\t, \mathbf{0}\t,\dots,
\mathbf{0}\t)\t$. Denote by $U_S$ the non-zero $s\times
p$-dimensional submatrix $({\beta_1^*}\t,\dots, {\beta_s^*}\t)\t$.

 Let $A(t) = A+ t(\ti\beta \hat\beta_i \t+\hat\beta_i \ti\beta \t)+t^2\ti\beta \ti\beta\t$ for some fixed $i$, with $\ti\beta  \in
 \fun{span}\{\hat\beta_1,\dots,\hat\beta_n\} = \fun{span}\{\hat\beta_1^*,\dots,\hat\beta_s^*\}$.
%

\marg{I removed the proof of the lemma on convexity of the norm,
and added this paragraph here, to replace the ref to the Lemma}

If $(x_k(t),c_k(t))$ is an eigen-pair of $A(t)$, then taking the
derivative of \(x_i\t x_i=1\) yields \(x_i\t\dot x_i=0\), and
trivially, since $x_i$ is an eigenvector, also \(x_i\t A\dot
x_i=0\). Here $\dot{}$ and $\ddot{}$  the first and second
derivative, respectively, according to $t$. Also, we have
 \eqsplit{
    x_k(t) &=x_k + tu_k+\o(t)
    \\
    c_k(t) &=c_k + t\nu_k+\o(t)
 }and
 \eqsplit{
    \Bigl(A +t(\ti\beta\hat\beta_i\t+\hat\beta_i\ti\beta\t)\Bigr)(x_k + t u_k) &= (c_k+t\nu_k)(x_k+t
    u_k)+\o(t),
  }
where $u_k\perp x_k$.

Equating the $\O(t)$ terms obtain
 \eqsplit{
    A u_k + (\ti\beta \hat\beta_i\t+\hat\beta_i\ti\beta\t)x_k &= c_k u_k + \nu_k x_k.
  }
Take now the inner product of both sides with $x_k$ to obtain that
 \eqsplit[nuk]{
    \nu_k=2(\ti\beta\t x_k)(x_k\t\hat\beta_i).
   } Note that the null space of $A(t)$ does not depend on $t$.
   Hence, if we call $\psi(\scb)=|||\scb|||_1$,
 \eqsplit{
    \frac{\partial}{\partial t  }  \psi(A(t))|_{t=0}
    &= \sum_{c_k>0} \frac{\partial}{\partial t  }c_k^{1/2}(t)|_{t=0}
    \\
    &= \frac12 \sum_{c_k> 0} \frac{\nu_k}{c_k^{1/2}}
    \\
    &=  \ti\beta\t \sum_{c_k> 0} c_k^{-1/2} x_kx_k\t \hat\beta_i
    \\
    &= \ti\beta\t A^{+1/2}\hat\beta_i = \ti\beta\t
    (\hat\scb\hat\scb\t)^{+1/2}\hat\beta_i\\
    &=  \ti\beta\t U_S\t (U_S\hat\scb\hat\scb\t U_S\t)^{-1/2} U_S\hat\beta_i,
} where $A^{+1/2}$ is the generalized inverse of $A^{1/2}$.

Taking, therefore, the derivative of the target function with
respect to $\hat\beta_i$ in the directions of
$\ti\beta\in\fun{span}\{\hat\beta_1,\dots,\hat\beta_n\}$ (e.g., in
the directions $\ti\beta=\beta_{\xi}^*$, $\xi=1,\dots,s$) gives
 \eqsplit{
    0 &= (\beta_{\xi}^*)\t ( -2 X_{i}\t (Y_i - X_i\hat\beta_i) + \lm (\hat\scb \hat\scb
    \t)^{+1/2}\hat\beta_i), \quad \text{or, equivalently,}\\
    \mathbf{0} &= U_S( 2 X_{i}\t (Y_i - X_i\hat\beta_i) - \lm   (\hat\scb \hat\scb \t)^{+1/2}
    \hat\beta_i).
  }
Let $R= (r_1,\dots,r_p)\t$ be the matrix of projected residuals:
 \eqsplit{
    R_{\ell i} = \summ j1m x_{ij\ell}(y_{ij}-x_{ij}\t\hat\beta_i), \quad \ell=1,\dots,p\,; \; i=1,\dots,n.
  }
Then
 \eqsplit{
    U_S R &=\frac\lm2 U_S (\hat\scb \hat\scb \t)^{+1/2}\hat\scb.
  }

Consider again the general expansion $\hat\scb=\summ \xi1{p\wedge
n} \alpha_\xi\beta_\xi^* {\be_\xi^*}\t$. Then
$|||\hat\scb|||_1=\summ\xi1{p\wedge n} |\al_\xi|$. Taking the
derivative of the sum of squares part of the target function with
respect to $\al_\xi$ we get
 \eqsplit{
    \summ i1n \be^*_{\xi i}{\beta^*_\xi}\t X_i\t(Y_{i}-X_{i}\hat\beta_i)
    &= {\beta_\xi^*}\t R \be_\xi^*.
  }
Considering the sub-gradient of the target function we obtain that
$|{\beta_\xi^*}\t R \be_\xi^*| \leq \lm/2$, and $\al_\xi=0$ in
case of strict inequality.

\end{proof}

\begin{proof}[Proof of Theorem \ref{th:BRTspectral} ]

(a) and (b) Similarly to the proof of Theorem~\ref{th:LassoL1p},
we have
 \eqsplit{
 \|Y-X\t\hat\scb\|_2^2  & = \|Y-X\t\scb\|_2^2
    +2\sum_{ij}\eps_{ij}x_{ij}\t(\beta_i -  \hat\beta_i).
  }

The last term can be bounded with high probability. Introduce
matrix $M$ with independent columns $M_i = X_i \varepsilon_i \sim
\scn_p(\mathbf{0}, m\sigma^2 I_p)$, $i=1,\dots,n$, since $\sum_j
x_{ij\ell}^2 = m$. Denote $q$-Schatten norm by $|||\cdot|||_q$.
Using the Cauchy-Swartz inequality and the equivalence between
$\ell_2$ (Frobenius) and Schatten with $q=2$ norms, we obtain:
 \eqsplit{
|\sum_{ij} \eps_{ij}x_{ij}\t(\beta_i -\hat\beta_i)| &=
|\sum_{i\ell} M_{i\ell} (\beta_{i\ell} -\hat\beta_{i\ell})|
\leqslant ||\scb-\hat\scb||_2\,  ||M||_{2} =
|||\scb-\hat\scb|||_2\, ||M||_{2}\\ &\leqslant
|||\scb-\hat\scb|||_1\, ||M||_{2}.
 }
 Now, $||M||_{2}^2 \sim m\sigma^2 \chi^2_{np}$ hence it can
be bounded by $B^2=m\sigma^2 (np + c)$ (Lemma A.1, Lounici et
al.~\cite{p:Lounici-GroupLasso}) with probability at least $1-
\exp\left(-\frac 1 8 \min(c, c^2/(np))\right)$. Denote this event
by \sca. Hence, we need to choose $c$ such that $c/\sqrt{np} \to
\en$. For example, we can take $c=A np$ with $A>1$, then $B= \sig
\sqrt{(1+A)mnp}$, and, since $\min(Anp, A^2np) = Anp$, the
probability is at least $1 - e^{-Anp/2}$.

Denote by $V$ the subspace of $\mathbb{R}^p$ corresponding to the
union of subspaces where the eigenvalues of $\scb \scb\t$ are
non-zero, and by $P_V$ the projection on that space. Then,
$\mathbb{R}^p = V \oplus V^c$ and $\dim(V) = \rank(\scb)\leqslant
s$.

Hence, adding $\lambda_2 |||\scb-\hat\scb|||_1$ to both sides, we
have that on \sca,
 \eqsplit{
 \|X\t(\scb-\hat\scb)\|_2^2 +\lm_2 |||\scb-\hat\scb|||_1 & \leqslant \lambda |||\scb|||_1 - \lambda |||\hat\scb|||_1 +(2 B+\lm_2) |||\scb-\hat\scb|||_1\\
 &\leqslant \lambda |||P_V \scb|||_1 -
\lambda \trace(P_V|\hat\scb| + (I-P_V)|\hat\scb|)\\ &+ (2 B+\lm_2)
|||P_V(\scb-\hat\scb)|||_1\\ &+(2 B+\lm_2) |||(I-P_V)(\scb-\hat\scb)|||_1\\
&\leqslant \lambda \trace(|P_V \scb|) - \lambda
\trace(P_V|\hat\scb|)
 + (2 B+\lm_2) \trace(|P_V(\scb-\hat\scb)|)\\ &+(2 B+\lm_2) \trace(|(I-P_V) \hat\scb|) - \lm \trace((I-P_V)|\hat\scb|)\\
&\leqslant (\lambda  +2 B+\lm_2) \trace(|P_V(\scb-\hat\scb)|),  }
if  $\lambda \geqslant 2B+\lambda_2$, since $ \trace(|P_V
\hat\scb|)=\trace(|P_V|\, |\hat\scb|)=\trace( P_V \, |\hat\scb|)$.
Here $|A| = (AA\t)^{1/2}$.
 We can take, e.g.
$\lambda_2 =2B =\lambda/2$, implying that $\lm =
4\sigma\sqrt{(1+A)mnp}$.

  Hence, we have that $\frac{\lm} 2  ||| \scb-\hat\scb||| \leqslant 2
\lm |||P_V(\scb-\hat\scb)|||$, i.e. $  |||(I-P_V)
(\scb-\hat\scb)||| \leqslant 3 \lm |||P_V(\scb-\hat\scb)|||$.
Thus, applying RE2$(s, 3,\kappa)$, $\rank(\scb)\leqslant s$, we
have that
 \eqsplit{
 \|X\t(\beta - \hat{\beta})\|_2^2 &\leqslant 2\lm |||P_V(\scb-\hat\scb)|||_1 \leqslant 2\lambda \sqrt{s} |||P_V(\scb-\hat\scb)|||_2\\
 &=  2\lambda \sqrt{s} ||P_V(\scb-\hat\scb)||_2 \leqslant \frac{2\lambda \sqrt{s}}{\kappa\sqrt{m}} ||X\t(\beta - \hat{\beta})||_2
 }
hence
 \eqsplit{
 \|X\t(\beta - \hat{\beta})\|_2  & \leqslant \frac{2\lambda
 \sqrt{s}}{\kappa\sqrt{m}}.
 }
Using this and the RE2 assumption,
 \eqsplit{
||| \scb-\hat\scb |||_1 \leqslant 4 |||P_V( \scb-\hat\scb) |||_1
\leqslant \frac{4\sqrt{s}}{\kappa \sqrt{m}} \|X\t(\beta -
\hat{\beta})\|_2 & \leqslant \frac{8\lambda
 s}{\kappa^2 m}.
 }
Substituting the value of $\lambda$, we obtain the results.

 (c) Since $\hat\gamma_i = \hat{U}\hat\beta_i$ are the
solution of group lasso problem with design matrices $\ti{X}_i =
\hat{U} X_i$, for $\ell \in J(\hat\gamma)$: $\|\hat\gamma_{\cdot
\ell}\|_2 \neq 0$, $\hat\gamma_{i\ell}$ satisfies the following
equations;
$$
 2 \tilde{X}_{i\cdot \ell}\t(Y_i  -X_i\hat\beta_i) =\lm
        \frac{\hat\gamma_{i\ell}}{  ||\hat\gamma_{\cdot
        \ell}||_2}
$$
(see also Theorem~\ref{th:sparseSpect}).

Hence,
$$
  \summ i1n \left(\tilde{X}_{i\cdot \ell}\t(Y_i  -X_i\hat\beta_i)\right)^2 =\frac{\lm^2}
  4.
$$

On one hand, for $\ell\in J(\hat\gamma)$, \eqsplit{ \left[ \summ
i1n \left(\tilde{X}_{i\cdot \ell} X_i\t(\hat\beta_i
-\beta_i)\right)^2\right]^{1/2} &\geqslant \left[ \summ i1n
\left(\tilde{X}_{i\cdot \ell} (Y_i
-X_i\t\hat\beta_i)\right)^2\right]^{1/2}\\
  &- \left[ \summ i1n \left(\tilde{X}_{i\cdot
\ell} (Y_i -X_i\t\beta_i)\right)^2\right]^{1/2}\\
&= \frac {\lm}{2} - \left( \summ i1n (U_\ell X_{i}
\varepsilon_i)^2\right)^{1/2}.}

On event \sca, \eqsplit{ \summ i1n (U_\ell X_{i} \varepsilon_i)^2
&= \summ i1n (U_\ell M_i)^2\leqslant \summ i1n ||U_\ell||_2^2
||M_i||^2 = |||M|||_2^2 \leqslant B^2 = (\lambda/4)^2. }


Summing over $\ell\in J(\hat\gamma)$, we have
 \eqsplit{
 \sum_{\ell\in J(\hat\gamma)}  \summ i1n \left(\tilde{X}_{i\cdot
\ell} X_i\t(\hat\beta_i -\beta_i)\right)^2  &\geqslant
\scm(\hat{\gamma}) \left( \frac{\lm}{2} - \frac{\lm}{4}\right)^2 =
\scm(\hat{\gamma}) \frac {\lm^2}{16} .}

 On the other hand, \eqsplit {
\summ \ell 1 s \summ i1n \left(\tilde{X}_{i\cdot \ell}
X_i\t(\hat\beta_i-\beta_i)\right)^2 &\leqslant \summ i1n
||\tilde{X}_{i} X_i\t(\hat\beta_i-\beta_i)||^2_2 =  \summ i1n
||X_{i} X_i\t(\hat\beta_i-\beta_i)||^2_2\\ &\leqslant m\phi_{\max}
|| X\t(\hat\scb-\scb)||^2_2. }

Since $\rank(\hat\scb) = \scm(\hat{\gamma})$,
$$
\rank(\hat\scb) \leqslant \frac{m\phi_{\max}
|| X\t(\hat\scb-\scb)||^2_2}{(\lm/4)^2} = \frac{16 m \phi_{\max}}{\lm^2} \frac{4\lm^2 s}{m \kappa^2} = s\, \frac{64 \phi_{\max}}{\kappa^2}.
$$

\end{proof}


\begin{proof}[Proof] of Theorem~\ref{th:persist}.

Using Lemma~\ref{lem:persist}, with probability at least $1 -
\eta$,
$$
| L_{F}(\beta) - L_{\hat{F}}(\beta) | \leqslant   \frac 1 {nm }
\sqrt{\frac{4e V   \log (np)} {m\eta}} (n+ \summ i1n
\|\beta_i\|_1^2),
$$
since $n>1$. Note that if $n=1$, it is sufficient to replace $p$
by $p+1$ under the logarithm.

In our case, the estimators are in set $B_{n,p}$. If $\summ i1n
\beta_i \beta_i\t = U\t \Lambda U$ is the spectral decomposition,
and $\gamma_i = U \beta_i$, $\Lambda_{kk} = ||\gamma_{\cdot
k}||_2^2$, $\gamma_{\cdot k}$ are orthogonal, hence
$$
\trace\{\summ i1n \beta_i \beta_i\t \}^{1/2} = \summ k1p
||\gamma_{\cdot k}||_2.
$$

Thus, we need to bound $\summ i1n  \|\beta_i\|_1^2 $ in terms of
$\summ k1p ||\gamma_{\cdot k}||_2$.

\marg{You have introduce a factor 2. I didn't see why, any way, the original expansion is after this one, with \% signs}\eqsplit {
\summ i1n  \|\beta_i\|_1^2
&\leqslant  \summ i1n M(\beta_i) \|\beta_i\|_2^2
\\
&= \max_i M(\beta_i) \summ i1n \|\gamma_i\|_2^2
\\
&= \max_i M(\beta_i) \summ \ell 1p \|\gamma_{\cdot \ell}\|_2^2
\\
&\leqslant 2
\max_i M(\beta_i) \left(\summ \ell 1p \|\gamma_{\cdot \ell}\|_2\right)^2
\\
&\leqslant \max_i M(\beta_i) b^2,
 }
since $\summ \ell 1p \|\gamma_{\cdot \ell }\|_2 \leqslant b$.

Hence, with probability at least $1 - \eta$,
 \eqsplit{
\sup_{F\in \calF } P_{F} \left( L_{F}\left(\hat\beta\right) -
L_{F}\left(\beta^*_{F} \right)\right)
 \leqslant    2\left(\frac 1
m + \frac{\max_i M(\beta_i) b^2}{nm}\right) \sqrt{\frac{4e V  \log (np)}
{m\eta} }.}
 Note that we can use $p$ instead of $\max_i M(\beta_i)$.
 The theorem is proved.

\end{proof}



\bibliographystyle{plain}
\bibliography{references}

\end{document}